\documentclass[10pt,twocolumn,letterpaper]{article}

\usepackage[pagenumbers]{cvpr} 

\usepackage{graphicx}
\usepackage{amsmath}
\usepackage{amssymb}
\usepackage{booktabs}

\usepackage[square,sort,comma,numbers]{natbib}
\usepackage{enumitem}
\usepackage[dvipsnames]{xcolor}
\usepackage{lipsum}
\usepackage{multirow}

\usepackage[export]{adjustbox}

\usepackage{afterpage}
\usepackage{bm}
\usepackage{forloop}

\newcommand{\inctabcolsep}[2]{\addtolength{\tabcolsep}{#1} #2 \addtolength{\tabcolsep}{-#1}}

\renewcommand{\paragraph}[1]{\vspace{.35em}\noindent\textbf{#1}}

\usepackage{amsthm}
\newtheorem{theorem}{Theorem}

\newcommand\blfootnote[1]{%
  \begingroup
  \renewcommand\thefootnote{}\footnote{#1}%
  \addtocounter{footnote}{-1}%
  \endgroup
}

\newcommand\arxivonly[1]{
\ifx\usearxiv\undefined
\else
#1
\fi
}
\def\usearxiv{1}

\usepackage[pagebackref,breaklinks,colorlinks]{hyperref}

\usepackage[capitalize]{cleveref}
\crefname{section}{Sec.}{Secs.}
\Crefname{section}{Section}{Sections}
\Crefname{table}{Table}{Tables}
\crefname{table}{Tab.}{Tabs.}

\def\cvprPaperID{6637} 
\def\confName{CVPR}
\def\confYear{2022}

\begin{document}

\title{\arxivonly{\vspace{-20px}}Disentangled Unsupervised Image Translation \\ via Restricted Information Flow}
\author{Ben Usman*\\
Boston University\\
{\tt\small usmn@bu.edu}
\and
Dina Bashkirova*\\
Boston University\\
{\tt\small dbash@bu.edu}
\and
Kate Saenko\\
Boston University\\
MIT-IBM Watson AI Lab \\
{\tt\small saenko@bu.edu}
}
\maketitle

\begin{abstract}
  Unsupervised image-to-image translation methods aim to map images from one domain into plausible examples from another domain while preserving structures shared across two domains. 
  In the many-to-many setting, an additional guidance example from the target domain is used to determine domain-specific attributes of the generated image.
   In the absence of attribute annotations, methods have to infer which factors are specific to each domain from data during training. Many state-of-art methods hard-code the desired shared-vs-specific split into their architecture, severely restricting the scope of the problem.
   In this paper, we propose a new method that does not rely on such inductive architectural biases, and infers which attributes are domain-specific from data by constraining information flow through the network using translation honesty losses and a penalty on the capacity of domain-specific embedding.
   We show that the proposed method achieves consistently high manipulation accuracy across two synthetic and one natural dataset spanning a wide variety of domain-specific and shared attributes.
\end{abstract}
\begin{figure}[t]
\begin{center}
\vspace*{-20px}\includegraphics[width=\linewidth,trim=0 0.4in 5.8in 0,clip]{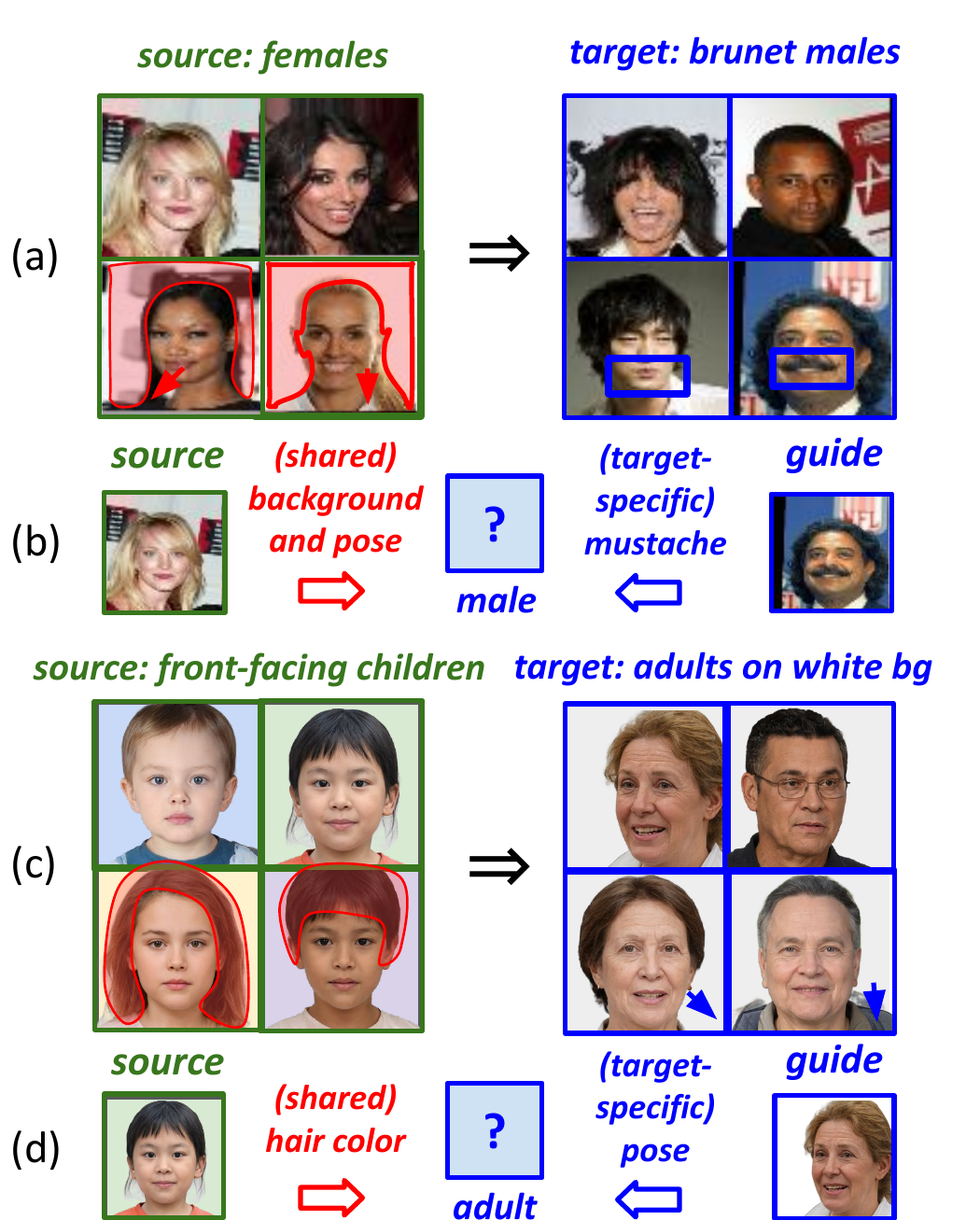}\vspace{-10px}
\end{center}
\caption{
\textbf{Problem.}
An unsupervised many-to-many image translation model must \textit{infer} which attributes are {\color{blue}target-specific}, \eg mustache in \textbf{(a)}, and pose in \textbf{(c)}, \textit{extract} values of these attributes from the ``guide'', and \textit{apply} them during cross-domain translation. Most state-of-art methods assume that all variations in colors and textures are target-specific, and fail to generalize when attributes like background texture, hair color or skin color are varied in both.
}\vspace{-15px}

\label{fig:fig1_problem}
\end{figure}

\section{Introduction}\label{sec:intro}

The goal \arxivonly{\blfootnote{* equal contribution}} of unsupervised image-to-image translation is to learn a mapping between two sets of images (or two domains) without any pair supervision. For example, face domains in Figure~\ref{fig:fig1_problem}a are defined by gender and share variability in poses and backgrounds, so a correct cross-domain mapping must change the face gender, but preserve the pose and the background of the original. When one domain has some unique attributes absent in the other domain, like males with and without beards in CelebA~\cite{CelebAMask-HQ}, the question of whether the ``correct'' one-to-one cross-domain mapping should add a beard to a specific female face does not have a well-defined answer. However, if we alter the problem definition by providing a ``guide'' image specifying male-specific factors, the resulting unsupervised \textit{many-to-many} translation problem has a well-defined correct solution: the learned mapping must preserve all attributes of the source image that are shared across two domains and use values of attributes unique to the target domain from the guide input (\eg preserve pose and background of the female source and add a beard from the male guide in the example above). 

\begin{figure}[t]
\begin{center}
\vspace{-15px}
\includegraphics[width=\linewidth,trim=0 2in 4in 0,clip]{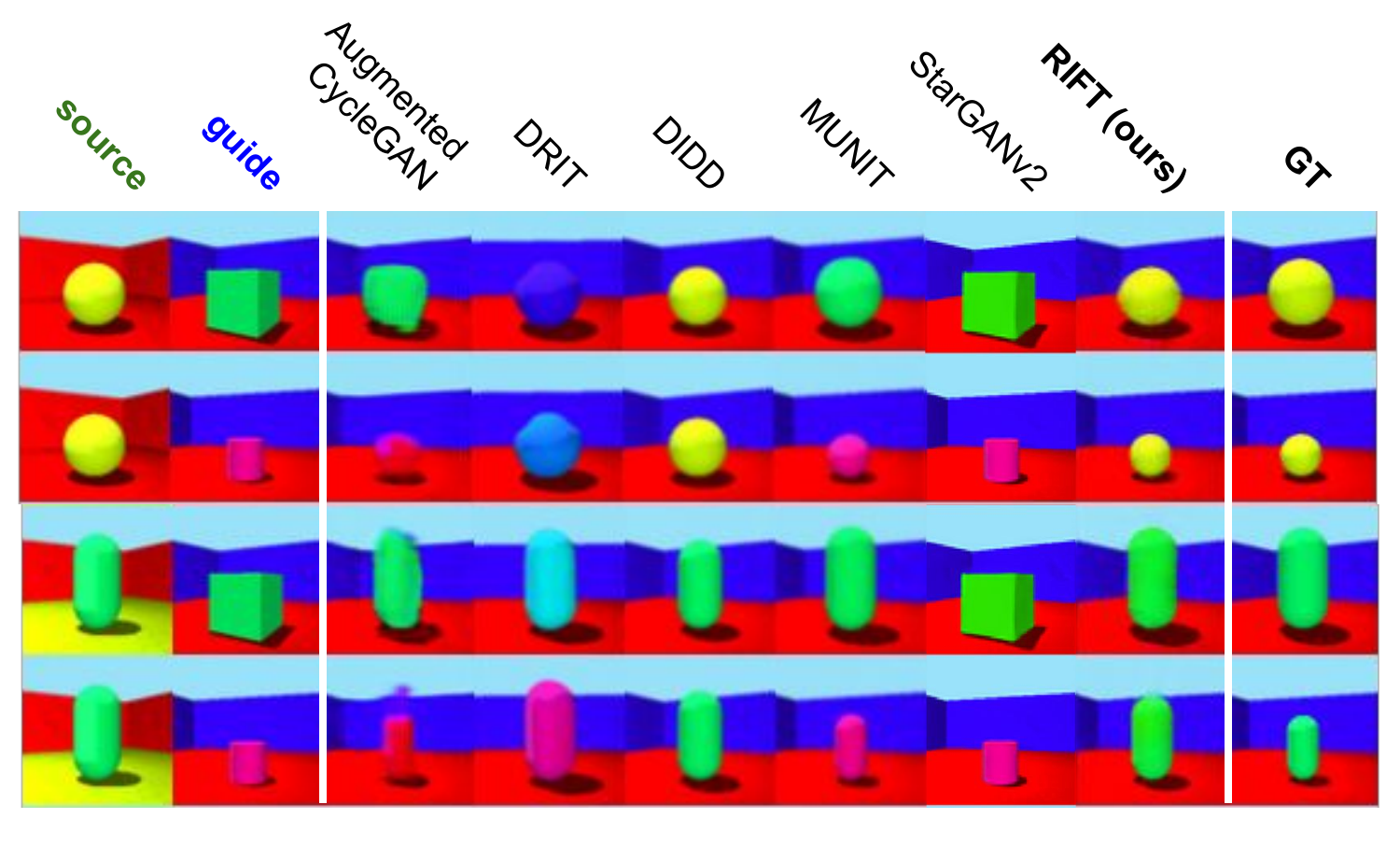}\vspace{-10px}
\end{center}
\caption{
\textbf{Flaws in existing methods.} 
On Shapes-3D-A \cite{kim2018disentangling,bashkirova2021evaluation}, all prior methods fail to either preserve shared attributes of the source (shape, object color), or apply target-specific attributes of the guide (size, orientation), or both. }\vspace{-10px}
\label{fig:shapes_compare}
\end{figure}

Most state-of-the-art methods for unsupervised many-to-many translation implicitly assume that the domain-specific variations can be modeled as ``global style'' (textures and colors) by hard-coding this assumption into their architectures via adaptive instance normalization (AdaIN) \cite{huang2017arbitrary} originally proposed for style transfer.
However, this choice severely restricts the kinds of problems that can be efficiently solved. 
More specifically, AdaIN-based methods~\cite{huang2018multimodal,choi2020stargan} inject domain-specific information from the guide image via a global feature re-normalization that forces colors, textures, and other global statistics to be always treated as domain-specific factors regardless of their actual distribution across the two domains. 
As a result, AdaIN-based methods change colors and textures of the input to match the guide image during translation even if colors/textures are varied across both domains and should not change. 
For example, background textures in the female-to-male setting (Figure~\ref{fig:fig1_problem}a)  vary in both domains, and therefore should be preserved, the same holds for hair color in the children-to-adults setting (Figure~\ref{fig:fig1_problem}c).
Even on a toy perfectly-balanced problem (Figure \ref{fig:shapes_compare}) AdaIN-based methods (\eg MUNIT \cite{huang2018multimodal}) change object color of the input to match the object color of the guide, even though object color is varied in both domains, and thus should be preserved.

Autoencoder-based methods \cite{almahairi2018augmented,DRIT_plus,benaim2019didd}, on the other hand, preserve shared information better, but often fail to apply correct domain-specific factors. For example, DIDD \cite{benaim2019didd} preserved the object color of the source in Fig.~\ref{fig:shapes_compare}, but failed to extract and apply the correct orientation and size from the guide. Overall, both our experiments and recent advances in evaluation of many-to-many image translation \cite{bashkirova2021evaluation}, show that all existing methods generally either fail to preserve global domain-specific attributes or fail to apply domain-specific factors well.

In this paper, we propose Restricted Information Flow for Translation \textbf{(RIFT)} - a novel approach that does not rely on an inductive bias provided by AdaIN and achieves high attribute manipulation accuracy across different kinds of attributes regardless of whether they are shared or domain-specific. As illustrated in Figure~\ref{fig:method_overview}, during ``brunet male-to-female translation'' our method preserves \textit{shared factors} (background and pose) of the input male face, and encodes \textit{male-specific attributes} (mustache) in a domain-specific embedding to enable accurate reconstruction of the source image. The core observation at the heart of our method is that \textit{only} values of shared attributes (background and pose) of the source can be encoded naturally in a generated image from the target domain, whereas source-specific attributes (mustache) can be encoded in the generated image only by ``hiding'' them in the form of structured adversarial noise \cite{bashkirova2019adversarial}. With this in mind, we propose using the \textit{translation honesty} loss \cite{bashkirova2019adversarial} to penalize the model for ``hiding''~\cite{chu2017cyclegan} a mustache inside the generated female image, and the \textit{embedding capacity} loss to penalize the model for encoding shared factors into the domain-specific embedding. As a result, information about the mustache is forced out of the generated female image into the domain-specific embedding, while information about the pose and background is forced out of domain-specific embeddings into the translation result - resulting in proper disentanglement of domain-specific and domain-invariant factors.

We measure how well RIFT models different kinds of attributes as either shared or domain-specific across three splits of Shapes-3D~\cite{kim2018disentangling}, SynAction \cite{sun2020twostreamvan} and Celeb-A \cite{CelebAMask-HQ} following an evaluation protocol similar to the one proposed by \citet{bashkirova2021evaluation}. Our experiments confirm that the proposed method achieves high attribute manipulation accuracy without relying on an inductive bias towards treating certain attribute kinds as domain-specific hard-coded into its architecture. 

\begin{figure}[t]
\begin{center}
\vspace{-15px}
\includegraphics[width=\linewidth,trim=0 2.7in 5.5in 0,clip]{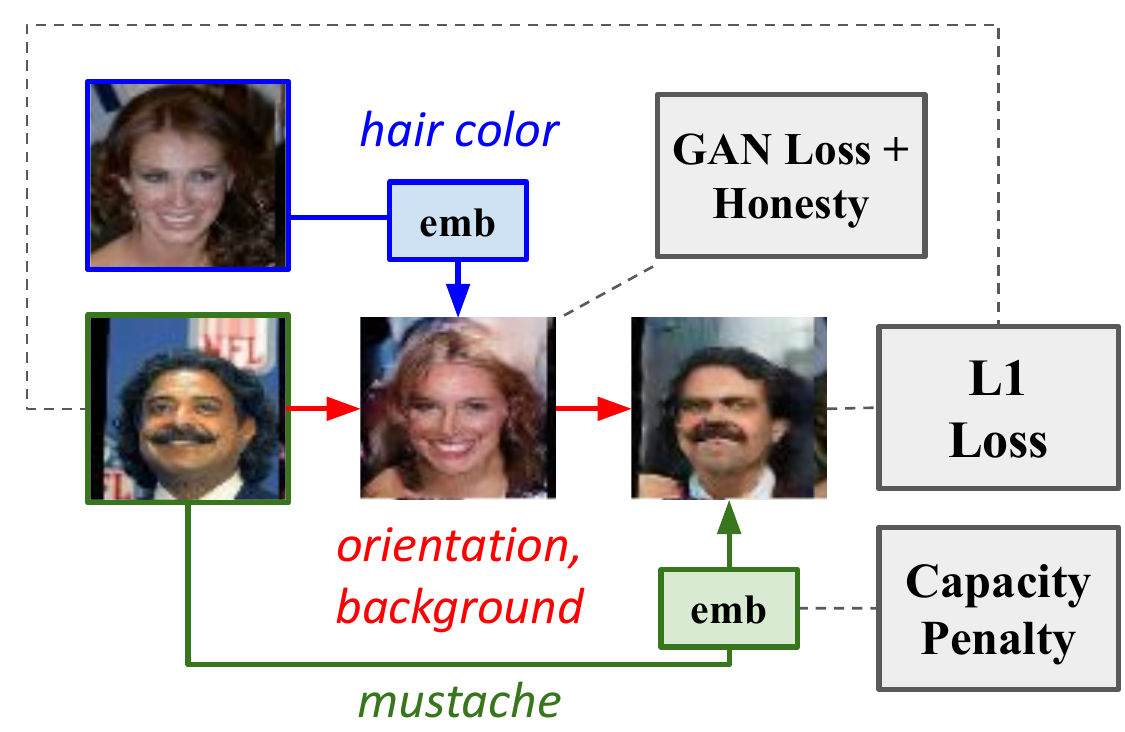}\vspace{-10px}
\end{center}
\caption{\textbf{Overview of the proposed method -- RIFT.} 
Source male image and male-specific factors ({\color{OliveGreen}green}), female guide input and female-specific attributes ({\color{blue}blue}), shared attributes ({\color{red}red}).}\vspace{-10px}
\label{fig:method_overview}
\end{figure}

\section{Related work}


\paragraph{Image-to-image translation.} In contrast to task-specific image translation methods \cite{cao2017unsupervised, guadarrama2017pixcolor, lugmayr2019unsupervised, qu2018unsupervised, gatys2015neural, ulyanov2016texture}, early unsupervised image-to-image translation methods, such as CycleGAN~\cite{zhu2017unpaired},
and UNIT~\cite{liu2017unsupervised},
infer semantically meaningful cross-domain mappings from arbitrary pairs of semantically related domains without pair supervision. 
These methods assume one-to-one correspondence between examples in source and target domains, which makes the problem ill-posed if at least one of two domains has some unique domain-specific factors, as we discussed in Section~\ref{sec:intro}. 

\paragraph{Many-to-many translation.} To account for such domain-specific factors, and to enable control over them in the translation results, many-to-many image translation methods \cite{huang2018multimodal,almahairi2018augmented,choi2020stargan,liu2019few,DRIT_plus} have been proposed. These methods separate domain-invariant ``content'' from domain-specific ``style'' using separate encoders. Following \cite{bashkirova2021evaluation}, we avoid terms ``content'' and ``style'' to distinguish general many-to-many translation from its subtask - style transfer \cite{gatys2015neural}.

\paragraph{Adaptive instance normalization.} Many state-of-art many-to-many translation methods, such as  MUNIT~\cite{huang2018multimodal}, FUNIT~\cite{liu2019few} and StarGANv2~\cite{choi2020stargan}, use AdaIN \cite{huang2017arbitrary}, originally proposed for style transfer \cite{gatys2015neural}.
More specifically, these methods modulate activations of the decoder with the domain-specific embedding of the guide. This architectural choice was shown to limit the range of applications of these methods to cases when 
domain-specific information lies within textures and colors~\cite{bashkirova2021evaluation}. 

\paragraph{Autoencoders.}
In contrast, methods like Augmented CycleGAN~\cite{almahairi2018augmented}, DRIT++~\cite{DRIT_plus} and Domain Intersection and Domain Difference (DIDD)~\cite{benaim2019didd} rely on embedding losses and therefore are more general.
For example, DIDD forces domain-specific embeddings of opposite domain to be zero, 
while DRIT++ uses adversarial training to make the source and target content embeddings indistinguishable. 

\paragraph{Cycle losses.} 
Most methods \cite{huang2018multimodal,almahairi2018augmented} also use cycle-consistency losses on domain-specific embeddings to ensure that information extracted from the guidance image is not ignored during translation, and cycle loss on images to improve semantic consistency \cite{chu2017cyclegan}.
However, cycle-consistency losses on images have been shown~\cite{chu2017cyclegan,bashkirova2019adversarial} to force one-to-one unsupervised translation models to ``cheat'' by hiding domain-specific attributes into translations.


Overall, prior methods ensure that the guide input modulates the translation result in some non-trivial way, but, to our knowledge, no prior work explicitly address adversarial embedding of domain-specific information into the translated result, or ensures that domain-invariant factors are preserved during translation, and this work fills this gap.

\section{Restricted Information Flow for Translation}\label{sec:method}

In this section we first formally introduce the many-to-many image translation problem, and describe how our method solves it. Our model reconstructs input images from translation results and domain-specific embeddings as illustrated in Fig.~\ref{fig:method_overview}, forcing domain-invariant information out of domain-specific embedding using capacity losses, and forcing domain-specific information out from the generated translation using honesty losses. 

\paragraph{Setup.} 
Following \cite{huang2018multimodal}, we assume that we have access to two unpaired image datasets $A = \{a_i\}$ and $B = \{b_i\}$ that share some semantic structure, differ visually (\eg male and female faces with poses, backgrounds and skin color varied in both). In addition to that, each domain has some attributes that vary only within that domain, \eg only males have variation in the amount of facial hair and only females have variation in the hair color (like in Figure~\ref{fig:fig1_problem}). Our goal is to find a pair of guided cross-domain mappings $F_{\text{A2B}}: A, B \to B$ and $F_{\text{B2A}}: B, A \to A$ such for any source inputs $a_s, b_s$ and guide inputs $a_g, b_g$ from respective domains, resulting guided cross-domain translations $b' = F_{\text{A2B}}(a_s, b_g)$ and $a' = F_{\text{B2A}}(b_s, a_g)$ look like plausible examples of respective output domains, share domain-invariant factors with their ``source'' arguments ($a_s$ and $b_s$ respectively) and domain-specific attributes with their ``guidance'' arguments ($b_g$ and $a_g$ respectively). This general setup covers the absolute majority of real-world image-to-image tasks. For example, the correct guided female-to-male mapping $F_{\text{B2A}}$ applied to female source image $b_s$ and a guide male image $a_g$ should generate a new male image $a'$ with pose, background, skin color, and other shared factors from the female input image $b_s$, and facial hair from the guidance input $a_g$, because poses, backgrounds and skin color vary in both, while facial hair is male-specific.
\begin{figure*}[t]
\begin{center}
\vspace{-15px}
\includegraphics[width=\linewidth,trim=0 2.8in 0.5in 0,clip]{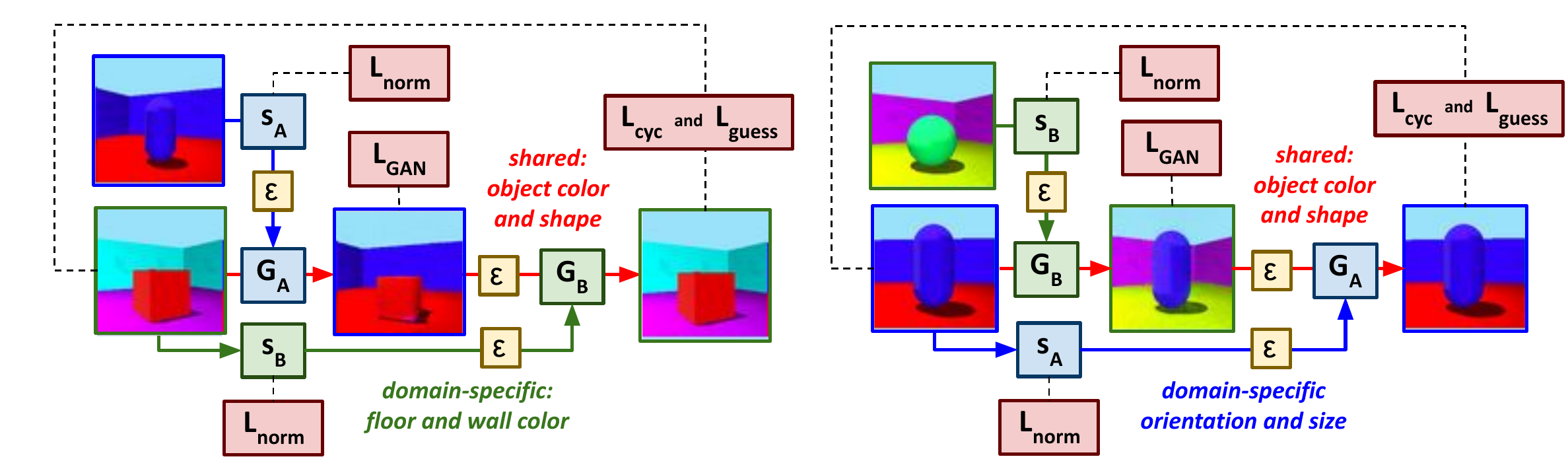}\vspace{-10px}
\end{center}
\caption{
\textbf{Losses used to train RIFT.} For illustration purposes, we use 3D-Shapes-A split described in Section~\ref{sec:experiment} and illustrated in Figure~\ref{fig:shapes_examples}. When the model is trained, {\color{OliveGreen}green} arrows carry only B-specific information (floor and wall color), {\color{blue}blue} arrows carry only A-specific information (orientation and size), and {\color{red}red} arrows carry information shared across two domains (object color and shape). }\vspace{-10px}
\label{fig:method}
\end{figure*}

\paragraph{Method.} While it might be possible to approximate functions $F_{\text{A2B}}$ and $F_{\text{B2A}}$ directly, following prior work, we split each one into two learnable parts:
encoders $s_A(a), s_B(b)$ that extract domain-specific information from corresponding guide images, and generators $G_{\text{A2B}}(a, s_b)$ and $G_{\text{B2A}}(b, s_a)$ that combine that domain-specific information with a corresponding source image, as illusrated in Figure~\ref{fig:method}. Final many-to-many mappings are just compositions of encoders and generators: 
\begin{gather*}
    F_{\text{A2B}}(a, b) = G_{\text{A2B}}(a, s_B(b)), \ 
    F_{\text{B2A}}(b, a) = G_{\text{B2A}}(b, s_A(a))
\end{gather*}

Our goal is to ensure that encoders $s_*$ extract all domain-specific information
from their inputs (and nothing more), and that generators $G_*$ use that information, along with domain-invariant factors from their source inputs to form plausible images from corresponding domains.

\paragraph{Noisy cycle consistency loss.} First, to ensure that each attribute of input images is not ignored completely, (\ie that it is treated as either domain-specific, or domain-invariant, or both), we use a guided analog of the cycle consistency loss. This loss ensures that any image translated into a different domain, and translated back with its original domain-specific embedding is reconstructed perfectly. Additionally, as the first step towards restricting the amount of information passed through each branch, we add zero-mean Gaussian noise ($\varepsilon$) of amplitude $\sigma_s$ or $\sigma_g$ and appropriate shape to translations and domain-specific embeddings respectively, before reconstructing images back:
\begin{gather*}
    L_{\text{cyc}}^A = \mathbb E_{a, b} \ ||a_\text{cyc} - a||_1, \ \ 
    L_{\text{cyc}}^B = \mathbb E_{b, a} \ ||b_\text{cyc} - b||_1 \\
    a_\text{cyc} = G_{\text{B2A}}(G_{\text{A2B}}(a, s_B(b) + \varepsilon_{g}) + \varepsilon_{s}, s_A(a) + \varepsilon_{g}) \\
    b_\text{cyc} = G_{\text{A2B}}(G_{\text{B2A}}(b, s_A(a) + \varepsilon_{g}) + \varepsilon_{s}, s_B(b) + \varepsilon_{g} )\\
    a \sim A, \ b \sim B, \ \varepsilon_{s} \sim \mathcal N(0, \sigma_{s}), \ \varepsilon_{g} \sim \mathcal N(0, \sigma_{g})
\end{gather*}
\paragraph{Translation honesty.} Unfortunately, any form of cycle loss encourages the model to ``hide'' domain-specific information inside the translated image in the form of structured adversarial noise~\cite{chu2017cyclegan}. 
To actively penalize the model for ``hiding'' the domain-specific information, such as mustache, inside a generated female image (instead of putting it into a male-specific embedding $s_a$), we use the \textit{guess loss} \cite{bashkirova2019adversarial}. 
This loss detects and prevents this so-called ``self-adversarial attack'' in the generator by training an additional discriminator to ``guess'' which of its two inputs is a cycle-reconstruction and which is the original image. 
For example, if the male-to-female generator $G_{\text{A2B}}$ is consistently adversarially embedding mustaches into all generated female images, then the cycle-reconstructed female $b_\text{cyc}$ will also have traces of an embedded mustache, and will be otherwise be identical to the input $b$. 
In this case, the guess discriminator, trained specifically to detect differences between input images and their cycle-reconstructions, will detect this hidden signal and penalize the model:  
\begin{gather*}
    L_{\text{guess}}^A = 
            [D^{\text{gs}}_A(a, a_\text{cyc})]^2 + [1 - D^{\text{gs}}_A(a_\text{cyc}, a)]^2 \\
    L_{\text{guess}}^B = 
            [D^{\text{gs}}_B(b, b_\text{cyc})]^2 + [1 - D^{\text{gs}}_B(b_\text{cyc}, b)]^2
\end{gather*}
\paragraph{Domain-specific channel capacity.} Unfortunately, neither of two losses described above can prevent the model from learning to embed the entire guide image $a_g$ into the domain-specific embeddings $s_a$ and reconstructing it from that embedding in $G_\text{B2A}$,  ignoring its first argument completely, \ie just always producing the guide input exactly. In order to prevent this from happening we penalize norms of domain-specific embeddings, effectively constraining the capacity of the resulting channel:
\begin{gather*}
L_{\text{norm}}^A = \mathbb E_a \ ||s_A(a)||_2^2, \ \
L_{\text{norm}}^B = \mathbb E_b \ ||s_B(b)||_2^2
\end{gather*}

Intuitively, the mutual information between the input guide image $a_g$ and the predicted translation $a'$ corresponds to the maximal amount of information that an observer could learn about translations $a'$ by observing guides $a_g$ if they had infinite amount of examples to learn from. Formally, using the derivation for the capacity of the additive white Gaussian noise channel (Sec. \ref{subsec:capacity}) we can show that:
\begin{gather*}
    \operatorname{MI}(a_g; a') \lesssim \operatorname{dim}(s_A(a)) \cdot \log_2 \left(1 + L^A_{\text{norm}} / \sigma_g^2\right), \\
    \text{ where } a' = G_{\text{B2A}}(b_s, s_A(a_g) + \varepsilon_{g}), \ \varepsilon_{g} \sim \mathcal N(0, \sigma_{g})
\end{gather*}
meaning that minimizing $L^A_{\text{norm}}$ loss effectively limits the amount of information from the guide image $a_g$ that $G_\text{A2B}$ can access to generate $a'$, \ie the effective capacity of the domain-specific embedding. Note that disabling either the noise ($\sigma_g=0$) or the capacity loss ($L_\text{norm} \to \infty$) results in effectively \textit{infinite} capacity, so we need both. Intuitively, this bound describes the expected number of ``reliably distinguishable'' embeddings that we can pack into a ball of radius $\sqrt{L^A_{\text{norm}}}$ given that each embedding will be perturbed by Gaussian noise with amplitude $\sigma_g$. 


\paragraph{Realism losses.} Remaining losses are analogous to CycleGAN \cite{liu2017unsupervised} losses that ensure that output images lie within respective domains:
\begin{gather*}
    L_{\text{GAN}}^A = 
            [D_A(a)]^2 + \left[1 - D_A(G_{\text{B2A}}(b, s_A(a) + \varepsilon_{s}^a))\right]^2 \\
    L_{\text{GAN}}^B = 
            [D_B(b)]^2 + \left[1 - D_B(G_{\text{A2B}}(a, s_B(b) + \varepsilon_{s}^b))\right]^2 \\
    L_{\text{idt}}^A = \mathbb E_{a} \ ||G_{\text{B2A}}(a, s_A(a) + \varepsilon_{g}) - a||_1, \\
    L_{\text{idt}}^B = \mathbb E_{b} \ ||G_{\text{A2B}}(b, s_B(b) + \varepsilon_{g}) - b||_1
\end{gather*} 

\paragraph{Discriminator losses} We also train discriminator networks $D_A, D_B$ and guess discriminators $D_A^{\text{gs}}, D_B^{\text{gs}}$ by minimizing corresponding adversarial LS-GAN \cite{mao2017least} losses.

\section{Experiments}\label{sec:experiment}

We would like to measure how well each model can generalize across a diverse set of shared and domain-specific attributes. In this section we discuss datasets we used and generated to achieve this goal, as well as list baselines and metrics we used to compare our method to prior work.

\begin{figure}[t]
\begin{center}
\vspace{-10px}
\includegraphics[width=\linewidth,trim=0.1in 3.6in 5.9in 0,clip]{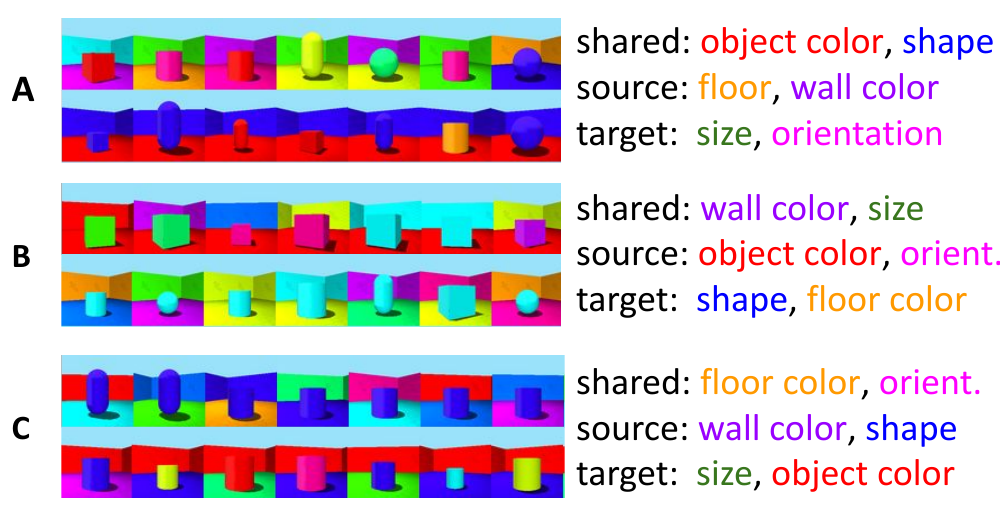}
\end{center}\vspace{-10px}
\caption{
\textbf{Shapes-3D-ABC} splits with respective shared and domain-specific attributes.}\vspace{-10px}
\label{fig:shapes_examples}
\end{figure}

\paragraph{Data.} 
Following the protocol proposed by \citet{bashkirova2021evaluation}, we re-purposed existing disentanglement datasets to evaluate the ability of our method to model different attributes as shared and domain-specific. We used 3D-Shapes \cite{kim2018disentangling}, SynAction \cite{sun2020twostreamvan} and CelebA \cite{CelebAMask-HQ}. Unfortunately, among the three, only 3D-Shapes \cite{kim2018disentangling} is balanced enough and contains enough labeled attributes to make it possible to generate and evaluate all methods across several attribute splits of comparable sizes. For example, if we attempted to build a split of SynAction with domain-specific pose attribute, the domain with fixed pose would only contain 90 unique images, which is not sufficient to train an unsupervised translation network.  

\paragraph{3D-Shapes-{ABC}.} The original 3D-Shapes \cite{kim2018disentangling} dataset contains 40k synthetic images labeled with six attributes: floor, wall and object colors, object shape and object size, and orientation (viewpoint). There are ten possible values for each color attribute, four possible values for the shape (cyliner, capsule, box, sphere), fifteen values for orientation, and eight values for size. We used three subsets of 3D-Shapes with different attribute splits visualized in Figure~\ref{fig:shapes_examples}. Three resulting domain pairs contained 4.8k/4k, 12k/3.2k, and 12k/6k images respectively.

\paragraph{SynAction.} We used the same \cite{bashkirova2021evaluation} split of SynAction~\cite{sun2020twostreamvan} - with background varied in one domain (nine possible values), identity/clothing varied in the other (ten possible values), and pose varied in both (real-valued vector). The resulting dataset contains 5k images in one domain and 4.6k images in the other. 
We note that the attribute split of this dataset \textbf{matches} the inductive bias of AdaIN methods, since the layout (pose) is shared and textures (background, clothing) are domain-specific in both domains.

\paragraph{CelebA.} We used the male-vs-female split proposed by~\cite{bashkirova2021evaluation} with 25k/25k images, and evaluated disentanglement of six most visually prominent attributes: pose, skin and background color (shared attributes, real-valued vectors), male-specific presence of facial hair (binary), female-specific hair color (three possible values), and domain-defining gender.

\paragraph{Baselines} We compare the proposed method against several state-of-art AdaIN methods, namely MUNIT \cite{huang2018multimodal}, StarGANv2 \cite{choi2020stargan}, MUNITX \cite{bashkirova2021evaluation}, and autoencoder-based methods, namely Domain Intersection and Domain Difference (DIDD) \cite{benaim2019didd}, Augmented CycleGAN \cite{almahairi2018augmented} and DRIT++ \cite{DRIT_plus}. In what follows we also provide a random baseline (RAND) that corresponds to selecting and returning a random image from the target domain. 

\paragraph{Metrics} In order to evaluate the performance of our method, we measured how well the domain-specific attributes were manipulated and domain-invariant attributes were preserved. Following \citet{bashkirova2021evaluation} we trained an attribute classifier $f(x)$, and for each attribute $k$, we measured the its \textit{manipulation accuracy} - the probability of correctly modifying an attribute across input-guide pairs for which the value of the attribute \emph{must change}:
\begin{equation*}
    \text{ACC}_k^{\text{A}} = p(f_k(F_{\text{A2B}}(a, b)) = y^*_k \ | \ f_k(a) \neq f_k(b))
\end{equation*}
where the ``correct'' attribute value equals $y^*_k = f_k(a)$ for shared attributes, and $y^*_k = f_k(b)$ otherwise. For real-valued multi-variate attributes (pose keypoints, background RGB, skin RGB, etc.) we measured the probability of generating an image with predicted attribute vector closer to the correct attribute vector $y^*_k$ then to the incorrect vector:
%
\begin{equation*}
    \text{ACC}_k^{\text{A}} = p(\|f_k(F_{\text{A2B}}(a, b)) - y^*_k\| \leq \|f_k(F_{\text{A2B}}(a, b)) - y'_k\|)
\end{equation*}
where $y^*_k = f_k(a)$ and $y'_k = f_k(b)$ for shared attributes, and vice-versa otherwise.
The manipulation accuracy in the opposite direction $\text{ACC}_k^{\text{B}}$ was estimated analogously. For \textit{Shapes-3D} we additionally \textit{aggregated} results across three splits by averaging manipulation accuracies across splits in which the given attribute was shared/common (C) or domain-specific (S). If we introduce the set of all splits $\mathcal S$ and predicates $\operatorname{common(k, s)}$ and $\operatorname{specific(k, s, \text{dom})}$, and the manipulation accuracy at a given split $\text{ACC}_k^A(s)$, \textit{aggregated manipulation accuracy} can be defined as follows:
\begin{gather}
    \text{ACC}_k^S = \frac{ \sum_{d \in \{\text{A, B}\}} \sum_{s \in S} \text{ACC}^d_k(s) \cdot \operatorname{specific(k, s, d)} } { \sum_{d \in \{\text{A, B}\}} \sum_{s \in S} \operatorname{specific(k, s, d)} } \\
    \text{ACC}_k^C = \frac{ \sum_{d \in \{\text{A, B}\}} \sum_{s \in S} \text{ACC}^d_k(s) \cdot \operatorname{common(k, s)} } { \sum_{d \in \{\text{A, B}\}} \sum_{s \in S} \operatorname{common(k, s)} }
\end{gather}
For three splits of \textit{3D-Shapes} we also report the \textit{relative discrepancy} between domain-specific and domain-invariant manipulation accuracies:
\begin{gather}
    \text{RD} = 100 \cdot \frac{ \sum_{k} |\text{ACC}^S_k -  \text{ACC}^C_k| }{ \sum_{k} (\text{ACC}^S_k +  \text{ACC}^C_k) }.
\end{gather}

\paragraph{Evaluation.} To compute metrics above, we generated two guided translations per source image per domain per baseline. We re-ran each method multiple times to account for poor initialization. We used PoseNet \cite{papandreou2018personlab} to get ground truth poses for SynAction, and \citet{ruiz2018headpose} and median background and skin color for CelebA, see suppl. Fig.~\ref{fig:clf_examples}.

\paragraph{Architecture.} We used standard CycleGAN components: pix2pix~\cite{isola2017image} generators and patch discriminators with LS-GAN loss \cite{mao2017least}. We archived best results when represented domain-specific embedding vectors as single-channel images, and 
made generators and encoders for the same domain (\eg $G_{\text{A2B}}$ and $s_A$) share all but last layers.

\section{Results}

In this section, we first compare our method to prior work both qualitatively and quantitatively. Then we show what happens if we remove key losses discussed Section~\ref{sec:method}. 
And finally, we discuss implicit assumptions made by our method, and key challenges that future methods might encounter in further improving manipulation accuracy across three datasets we used in this paper.

\begin{figure}[t]
\begin{center}
\vspace{-10px}
\includegraphics[width=\linewidth,trim=0.15in 0.1in 2.6in 0,clip]{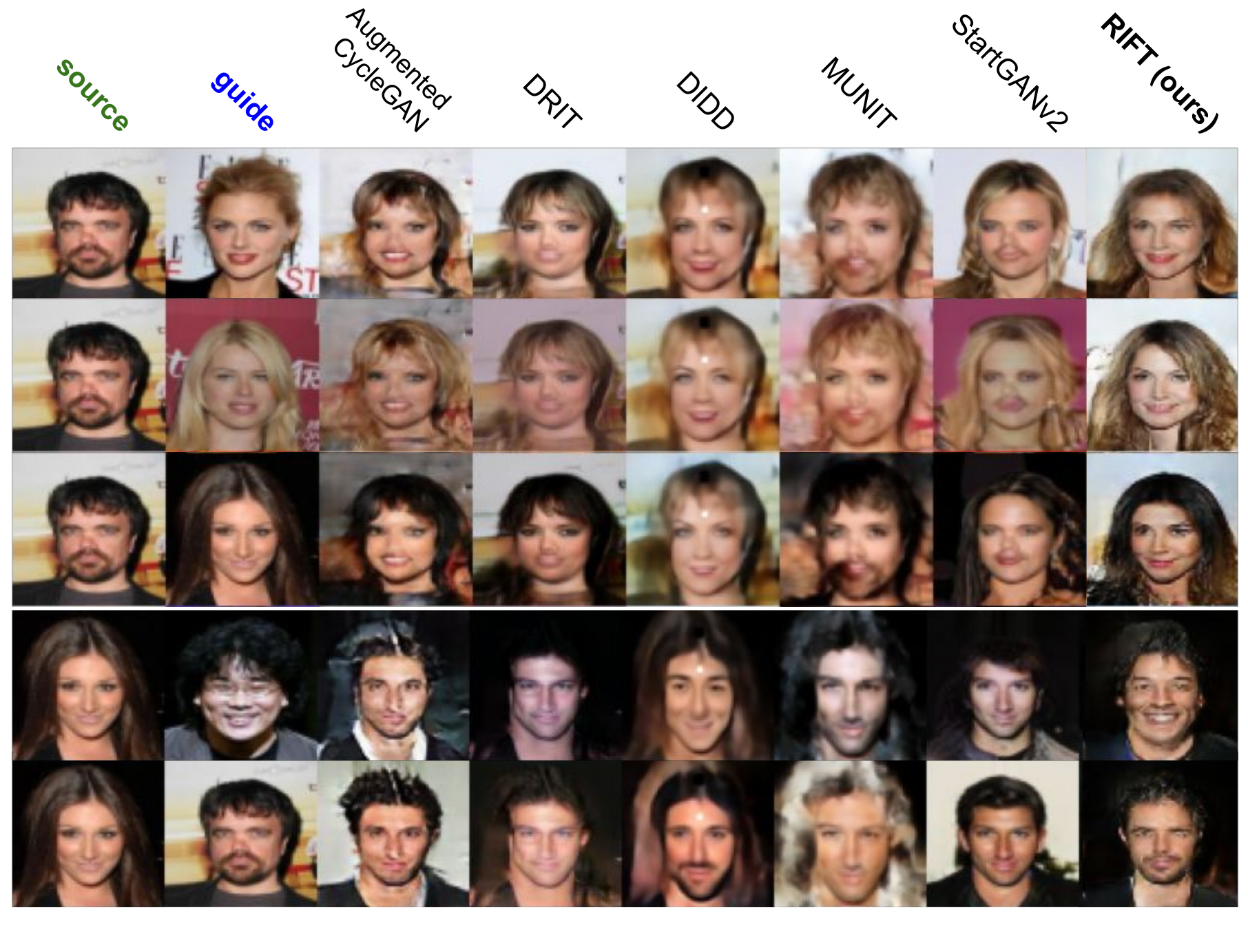}
\end{center}\vspace{-20px}
\caption{\textbf{Qualitative} results on CelebA. Methods should preserve pose and background of the source, and apply hair color of the female guide (top) and the facial hair of the male guide (bottom). }\vspace{-10px}
\label{fig:celeba_qual}
\end{figure}
\begin{table*}[ht]\centering\small
\vspace{-14px}
\inctabcolsep{-3.3pt} {
\begin{tabular}{lcccccccccccccc|cccc|ccccccc|l}\toprule
\multirow{3}{*}{\textbf{Method}} & \multicolumn{14}{c}{3D-Shapes\cite{kim2018disentangling}-ABC} & \multicolumn{4}{c}{SynAction \cite{sun2020twostreamvan}} & \multicolumn{7}{c}{CelebA \cite{CelebAMask-HQ} } &  \\
\cmidrule(lr){2-15} \cmidrule(lr){16-19} \cmidrule(lr){20-26} \cmidrule(lr){27-27} 
& \multicolumn{2}{c}{{FC}} & \multicolumn{2}{c}{{WC}} & \multicolumn{2}{c}{{OC}} & \multicolumn{2}{c}{{SZ}} & \multicolumn{2}{c}{{SH}} & \multicolumn{2}{c}{{ORI}} & \multicolumn{2}{c}{\textbf{AVG}} & \multicolumn{1}{c}{PS} & \multicolumn{1}{c}{IDT} & \multicolumn{1}{c}{BG} & \multicolumn{1}{c}{\textbf{AVG}} & HC & FH & GD & ORI & BG & SC & \multicolumn{1}{c}{\textbf{AVG}} & \multicolumn{1}{c}{\textbf{AVG}} \\ 
\cmidrule(lr){2-3} \cmidrule(lr){4-5} \cmidrule(lr){6-7} \cmidrule(lr){8-9} \cmidrule(lr){10-11} \cmidrule(lr){12-13} \cmidrule(lr){14-15} \cmidrule(lr){16-16} \cmidrule(lr){17-17} \cmidrule(lr){18-18} \cmidrule(lr){19-19} \cmidrule(lr){20-20} \cmidrule(lr){21-21} \cmidrule(lr){22-22} \cmidrule(lr){23-23} \cmidrule(lr){24-24} \cmidrule(lr){25-25} \cmidrule(lr){26-26} \cmidrule(lr){27-27} 
& C & S & C & S & C & S & C & S & C & S & C & S &
{\scriptsize{AC}$\uparrow$} & \multicolumn{1}{c}{{\scriptsize{RD}$\downarrow$}} 
& C & S & S 
& \multicolumn{1}{c}{{\scriptsize{AC}$\uparrow$}} 
& S & S & S & C & C & C 
& \multicolumn{1}{c}{{\scriptsize{AC}$\uparrow$}} & \multicolumn{1}{c}{{\scriptsize{AC$\uparrow \! \! \pm \! \sigma$}}}  \\
\midrule
StarGANv2 & 0 & \textbf{99} & 0 & \textbf{99} & 0 & 78 & 5 & \textbf{56} & 4 & \textbf{99} & 0 & \textbf{96} & 45 & 97 & \textbf{96} & \textbf{52} & \textbf{99} & \textbf{82} & \textbf{76} & 15 & 97 & 87 & 11 & 22 & 51 & 59{\scriptsize$\pm$20}\\
MUNIT & 5 & 94 & 0 & \textbf{99} & 0 & \textbf{97} & \textbf{59} & 31 & \textbf{96} & 58 & \textbf{99} & 61 & 58 & 56 & 75 & 28 & 7 & 37 & 45 & 7 & 90 & \textbf{89} & 43 & 44 & 53 & 49{\scriptsize$\pm$11}\\
MUNITX & 1 & 50 & 2 & 55 & 8 & 28 & 12 & 16 & 95 & 21 & \textbf{99} & 7 & 33 & 74 & 93 & 26 & 37 & 52 & 64 & 17 & 75 & 83 & 50 & 43 & 55 & 47{\scriptsize$\pm$12} \\
DRIT++ & 7 & 12 & 9 & 19 & 10 & 10 & 27 & 14 & 7 & 15 & 42 & 51 & 18 & \textbf{20} & 52 & 6 & 13 & 24 & 23 & 9 & 96 & \textbf{89} & 67 & 44 & 55 & 32{\scriptsize$\pm$20} \\
\scriptsize{AugCycleGAN} & 10 & 8 & 10 & 9 & 11 & 7 & 17 & 13 & 30 & 13 & 7 & 7 & 12 & \textbf{20} & 90 & 8 & 12 & 37  & 16 & 30 & 98 & 12 & 42 & 40 & 40 & 29{\scriptsize$\pm$15} \\
DIDD & 38 & \textbf{81} & 29 & 22 & 72 & 18 & 41 & 20 & 87 & 43 & 48 & 34 & 44 & 35 & 89 & 12 & \textbf{99} & 67 & 22 & \textbf{50} & 91 & 78 & \textbf{89} & 56 & \textbf{64} & 58{\scriptsize$\pm$12} \\
RIFT (ours) & \textbf{99} & 45 & \textbf{99} & 39 & \textbf{92} & 10 & 50 & 23 & 62 & 84 & 98 & 87 & \textbf{66} & 33 & 89 & 47 & \textbf{99} & 78 & 22 & 35 & \textbf{99} & 65 & 83 & \textbf{57} & 60 & \textbf{68{\scriptsize$\pm$9}}  \\
\midrule
RAND & 10 & 10 & 10 & 10 & 10 & 10 & 12 & 19 & 24 & 19 & 6 & 6 & 12 & 9 & 50 & 11 & 11 & 24 & 12 & 31 & 99 & 50 & 50 & 50 & 49 & 27{\scriptsize$\pm$20} \\
\bottomrule
\end{tabular}
}
\vspace{-5px}
\caption{\textbf{Manipulation accuracy} 
for six attributes aggregated across three splits of \textbf{Shapes-3D}: floor color (FC), wall color (WC), object color (OC), size (SZ), shape (SH), room orientation (ORI); three attributes in \textbf{SynAction}: pose (PS), identity/clothing (IDT), background (BG); and six attributes in \textbf{CelebA}: hair color (HC), facial hair (FH), gender (GD), face orientation (ORI), background (BG) and skin color (SC). We report per-attribute and average \textit{manipulation accuracies} for shared/common (C) or domain-specific (S) attributes, as well as overall \textit{average aggregated manipulation accuracy} (AC) and \textit{relative discrepancy} (RD) on 3D-Shapes described in Section~\ref{sec:experiment}. Table~\ref{tab:non_agg_3dsh} with non-aggregated performance across three splits of 3D-Shapes can be found in supplementary.} \vspace{-10px}
\label{tab:shape_results}
\end{table*}

\paragraph{Qualitative results.} Figures~\ref{fig:3dshapes_qual_grid} and \ref{fig:synaction_qual_grid} show that, in most cases, the proposed method successfully preserves domain-invariant content and applies domain-specific attributes from respective domains on 3D-Shapes and SynAction. Figure~\ref{fig:celeba_qual} shows that, on CelebA, our method preserves poses and backgrounds, and applies hair color better then other baselines. On 3D-Shapes-A, our method also preserves object color and applies correct size and orientation better than all alternatives (Figure~\ref{fig:shapes_compare}). A more detailed side-by-side qualitative comparison of generated images across all baselines and all datasets can be found in supp. Figures~\ref{fig:sup_example_1}-\ref{fig:sup_example_15}.

\paragraph{Quantitative results.} 
Table~\ref{tab:shape_results} shows that our method archives the highest average aggregate manipulation accuracy across three splits of 3D-Shapes, and second lowest relative discrepancy (RD) between accuracies of modeling same attributes as shared and specific across all attributes. On SynAction, that \textbf{matches} the inductive bias of AdaIN methods, our method performs on-par with AdaIN-based methods and outperforms all non-AdaIN methods. On CelebA, our method is much better at preserving background and skin colors then all AdaIN-based methods, and, in terms of the overall manipulation accuracy is second only to DIDD \cite{benaim2019didd}. Despite high manipulation scores, DIDD generated blurry faces and also struggled with applying correct hair color, as can be seen in Figure \ref{fig:celeba_qual} and supplementary. 
To sum up, our method achieves best or second-best performance in each of three dataset, and best performance overall (last column), with among the lowest discrepancy between per-attribute manipulation accuracies (RD), and lowest variance across datasets ($\pm\sigma$).

\begin{figure}[t]
\begin{center}
\vspace{-20px}
\includegraphics[width=\linewidth,trim=0.0in 2.4in 5.8in 0,clip]{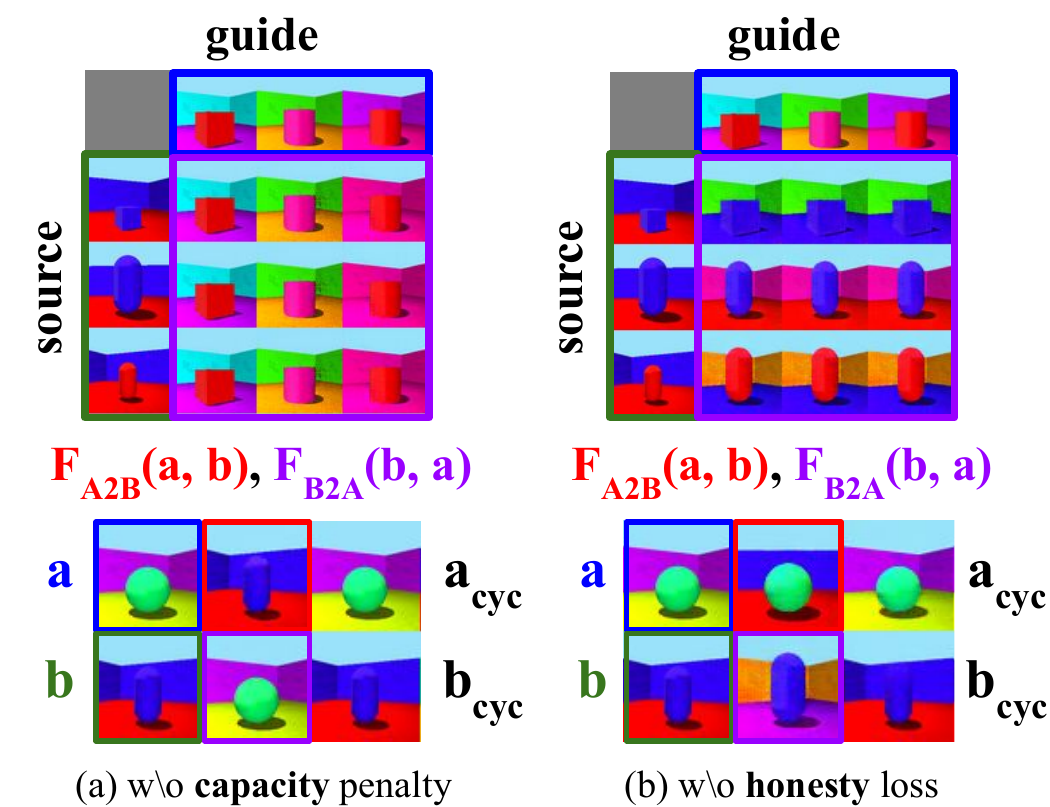}
\end{center}\vspace{-10px}
\caption{
\textbf{Ablations.}
Effects of disabling capacity and honesty losses on guided translations (top) and guided cycle-reconstructions (bottom) on Shapes-3D-A. Inputs images from domains {\bf\color{blue}A} and {\bf\color{OliveGreen}B}, {\bf\color{red}A2B} and {\bf\color{Plum}B2A} guided translations. 
}\vspace{-10px}
\label{fig:failure_modes}
\end{figure}
\paragraph{Ablations.} During B2A translation on Shapes-3D-A the model trained with all losses uses object color/shape from the source image and floor/wall color from the guide (Fig.~\ref{fig:3dshapes_qual_grid}). If we remove the penalty on the capacity of  domain-specific embeddings ($L_\text{norm}$), the model ignores the source input (Fig.~\ref{fig:failure_modes}a-top). The model encodes all attributes into domain-specific embeddings, and cycle-reconstructs inputs $a$ and $b$ perfectly from these embeddings (Fig.~\ref{fig:failure_modes}a-bottom), completely ignoring the source input: ${\color{OliveGreen}\bm{b}} = {\color{red}\bm{F_{\text{A2B}}(a, b)}} = b_\text{cyc}$. Removing honesty losses ($L_\text{guess}$), on the other hand, results in a model that ignores the guide input altogether (Fig.~\ref{fig:failure_modes}b-top). The model ``hides'' domain-specific information inside generated translations instead of the domain-specific embeddings, and makes domain-specific embeddings equal zero, resulting in zero capacity loss $L_\text{norm} = 0$, and zero cycle reconstruction loss $L_\text{cyc} = 0$. For example (Fig.~\ref{fig:failure_modes}b-bottom), the size and orientation of ${\color{OliveGreen}\bm{b}}$ is hidden inside ${\color{Plum}\bm{F_{\text{B2A}}(b, a)}}$ in the form of imperceptible adversarial noise and is used to reconstruct $b_\text{cyc}$ perfectly. If mapping $F_{\text{A2B}}$ actually used size and orientation of $b$ to generate $b_\text{cyc}$, it would have also applied that same size and orientation when generating ${\color{red}\bm{F_{\text{A2B}}(a, b)}}$, but it did not - so we conclude that both $F_{\text{A2B}}$ and $F_{\text{B2A}}$ ignore domain-specific embeddings and embed information inside generated translations instead. More illustrations in suppl. Figure~\ref{fig:sup_ablations}.

\paragraph{Challenges.}
We identified three major causes of remaining errors that existing methods fail to handle at the moment, and future researchers will need to address to make further progress in this task possible. First, some attributes ``affect'' very different number of pixels in training images, and as a consequence contribute very differently to reconstruction losses, making the job of balancing different loss components much harder. For example, the floor color in 3D-Shapes ``affects'' roughly half of all image pixels, whereas size affects only one tenth of all pixels - resulting in drastically different effective weights across all losses, especially if both are either domain-specific or shared at the same time. This explains highest performance of our method on 3D-Shapes-A (in comparison to 3D-Shapes-B,C, see Table~\ref{tab:non_agg_3dsh}) which has ``similarly-sized'' domain-specific attributes in both domains. The second challenge is that all datasets contain some attribute combinations that are almost distinguishable: for example, front-facing boxes and cylinders are hardly distinguishable in 3D-Shapes, and clothing of people is much less articulated when they are facing backwards because of shading in SynAction (see suppl. Figure~\ref{fig:challenges_bad_pairs}). This explains why our model fails on these cases: since it can not reliably reconstruct these attributes from such intermediate translations, it has no incentive to apply correct correct attributes values to them in the first place. Finally, unevenly distributed shared attributes in real world in-the-wild datasets (such as CelebA) pose even more serious challenge rendering the whole many-to-many problem setup not well defined. For example, if both male and female domains had hair color variation, but males were mostly brunet with only 3\% of blondes, and 50\% of females were blondes - should the model preserve blonde hair when translating females to males and sacrifice the ``realism'' of the generated male domain, or should it treat hair-color as a domain-specific attribute despite variations present in both? 

\paragraph{Ethical considerations.} While more precise attribute manipulation models requiring less supervision might be used for malicious deepfakes~\cite{nguyen2019deep,citron2018disinformation}, they can also be used to remove biases present in existing datasets~\cite{grover2019bias} to promote fairness in down-stream tasks \cite{augenstein2019generative}. We acknowledge that the CelebA dataset contains many biases (\eg being predominantly white) and that binary gender labels are problematic.


%


\begin{figure*}[!ph]
\begin{center}
\includegraphics[width=\linewidth,trim=0.3in 1in 0.6in 0,clip]{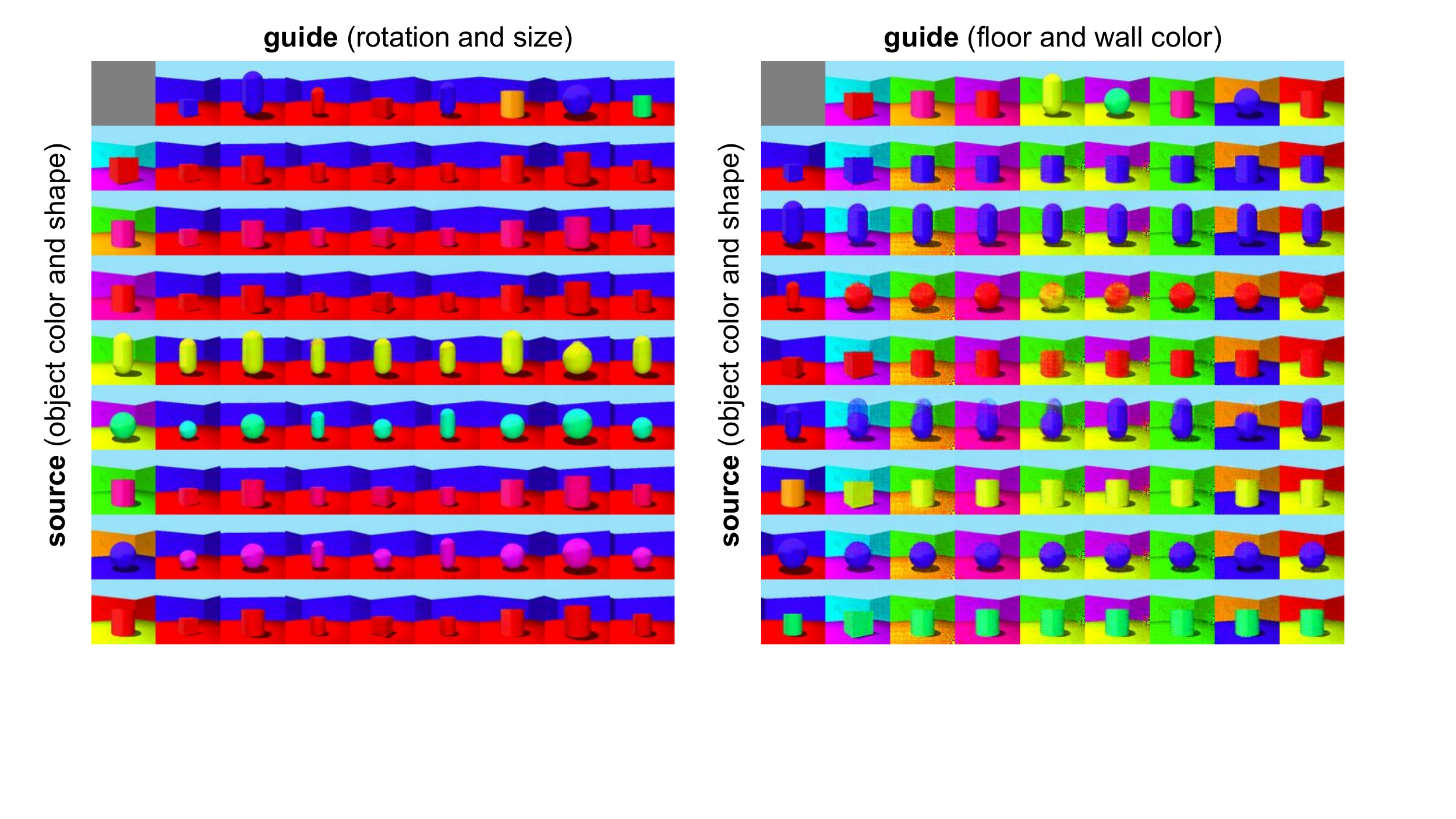}\vspace{-10px}
\end{center}
\caption{
\textbf{Guided translations generated by our method on 3D-Shapes-A.} Our model successfully preserves shared attributes (object color and shape) of the source image and applies domain-specific attributes of the guide domain (rotation and size on the left, floor and wall color on the right) in most cases. It sometimes confuses boxes with cylinders, as discussed in \textbf{challenges} paragraph.}\vspace{-10px}
\label{fig:3dshapes_qual_grid}

\begin{center}
\vspace{30px}
\includegraphics[width=\linewidth,trim=0.3in 1in 0.5in 0,clip]{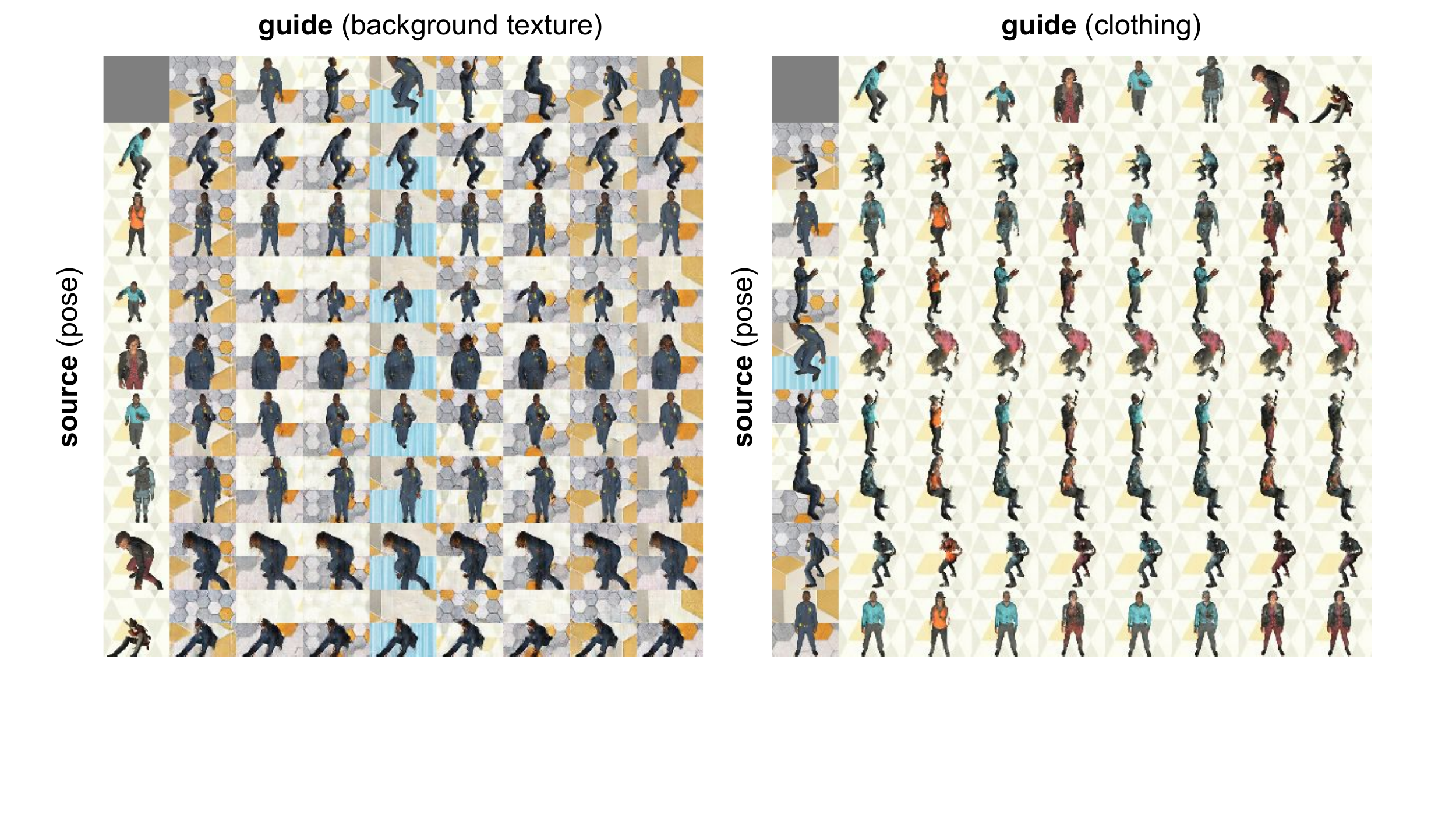}\vspace{-10px}
\end{center}
\caption{
\textbf{Guided translations generated by our method on SynAction.} Our model correctly preserve shared attributes (pose) of the source image and applies domain-specific attributes of the guide domain (background texture on the left, clothing/identity colors on the right). It sometimes applies wrong clothing, especially in extreme poses, as discussed in \textbf{challenges} paragraph.}\vspace{-10px}
\label{fig:synaction_qual_grid}
\end{figure*}

\section{Conclusion}

In this paper we propose RIFT - a new unsupervised many-to-many image-to-image translation method that does not rely on an inductive bias hard-coded into its architecture to determine which attributes are shared and which are domain-specific, and achieves consistently high attribute manipulation accuracy across a wide range of datasets with different kinds of domain-specific and shared attributes, and low discrepancy between manipulation accuracies across different attributes and datasets. Moreover, on datasets that match the inductive bias of AdaIN-based methods, the proposed method performs on-par with AdaIN-based methods. Finally, in this paper we identified three core challenges that need to be resolved to enable further development of unsupervised many-to-many image-to-image translation.

\clearpage

{\small
\bibliography{egbib}
}

\clearpage
\section{Supplementary}
\subsection{Derivation of the capacity}\label{subsec:capacity}
Let $A$ and $B$ be arbitrary datasets, $s$ and $G$ be domain-specific embedding and generator functions, and $a'$ be the translation from source $b$ to domain $A$, guided by the target example $a$. The following theorem bounds the amount of information about $a$ that $G$ can access to generate $a'$.
\begin{theorem}
The effective capacity of the guided embedding, \ie the capacity of the $a \to a'$ channel, \ie the mutual information $\operatorname{MI}(a; a')$ is bounded by:
\begin{gather*}
    \operatorname{MI}(a; a') \lesssim \operatorname{dim}(s(a)) \cdot \log_2 \left(1 + L / \sigma^2\right), \\
    \text{ where } a' = G(b, s(a) + \varepsilon), \ \varepsilon \sim \mathcal N(0, \sigma^2), \\
    \text{ and } L = \mathbb E \| s(a) \|_2^2 , \ a \sim A, \ b \sim B
\end{gather*}
\end{theorem}
\begin{proof}
First, let us define a Markov chain
$$
a \to s(a) \to (s(a) + \varepsilon) \to a'
$$
using the data processing inequality twice we can show that
$$
\operatorname{MI}(a; a') \leq \operatorname{MI}(a; s(a) + \varepsilon) \leq \operatorname{MI}(s(a); s(a) + \varepsilon)
$$
intuitively meaning that the overall pipeline always looses at least as much information as each of its steps. Then expanding the mutual information in terms of the differential entropy $h(X)$ gives us
\begin{gather*}
    \operatorname{MI}(s(a); s(a) + \varepsilon) = h(s(a) + \varepsilon) - h(s(a) + \varepsilon | s(a)) \\
     = h(s(a) + \varepsilon) - h(\varepsilon)
\end{gather*}
Since the the second raw moment (aka power) of $s(a)$ is bounded by $L$, the entropy $h(s(a) + \varepsilon)$ will be maximized if $s(a)$ is a $k$-dimensional spherical multivariate normal with variance $L$, where $k=\operatorname{dim}(s(a))$ therefore
\begin{gather*}
    \operatorname{MI}(s(a); s(a) + \varepsilon) \leq h(\mathcal N_k(0; L+\sigma^2)) + h(\mathcal N_k(0; \sigma^2)) \\
    = \frac{1}{2} \ln\left(\frac{(L + \sigma^2)^k}{\sigma^{2k}}\right) \leq k \cdot \log_2\left(1 + L / \sigma^2 \right).
\end{gather*}
\end{proof}

\begin{figure*}
\begin{center}
\includegraphics[width=\linewidth,trim=0 0.3in 0 0,clip]{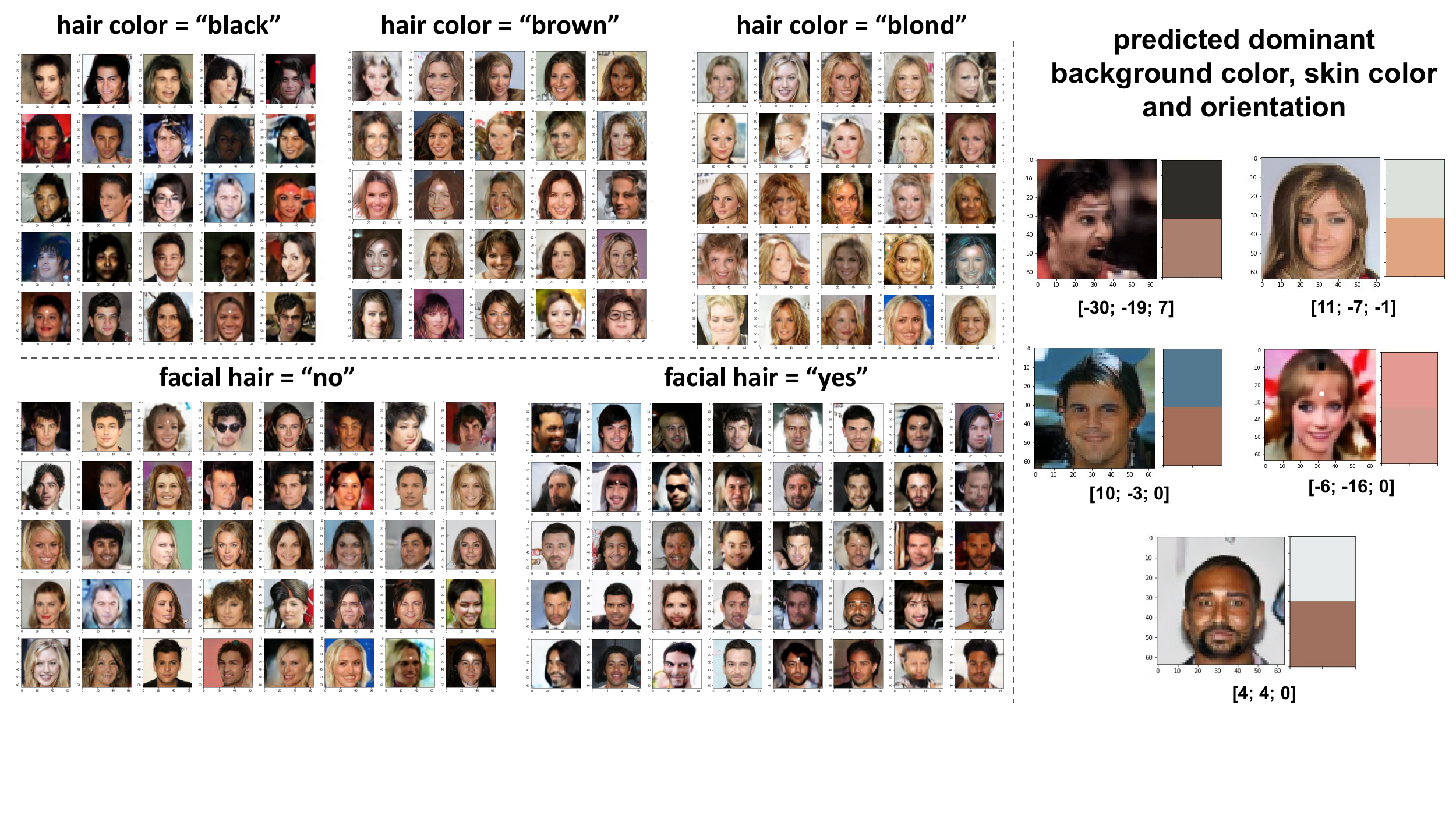}
\end{center}
\caption{Predictions of the attribute regression network on network predictions on CelebA for hair color (black, brown or blonde), facial hair (binary), background color (RGB), skin color (RGB), and face orientation (yaw; pitch; roll). }\label{fig:clf_examples}
\end{figure*}
\begin{figure*}
\begin{center}
\includegraphics[width=\linewidth,trim=0 1in 0 0,clip,page=1]{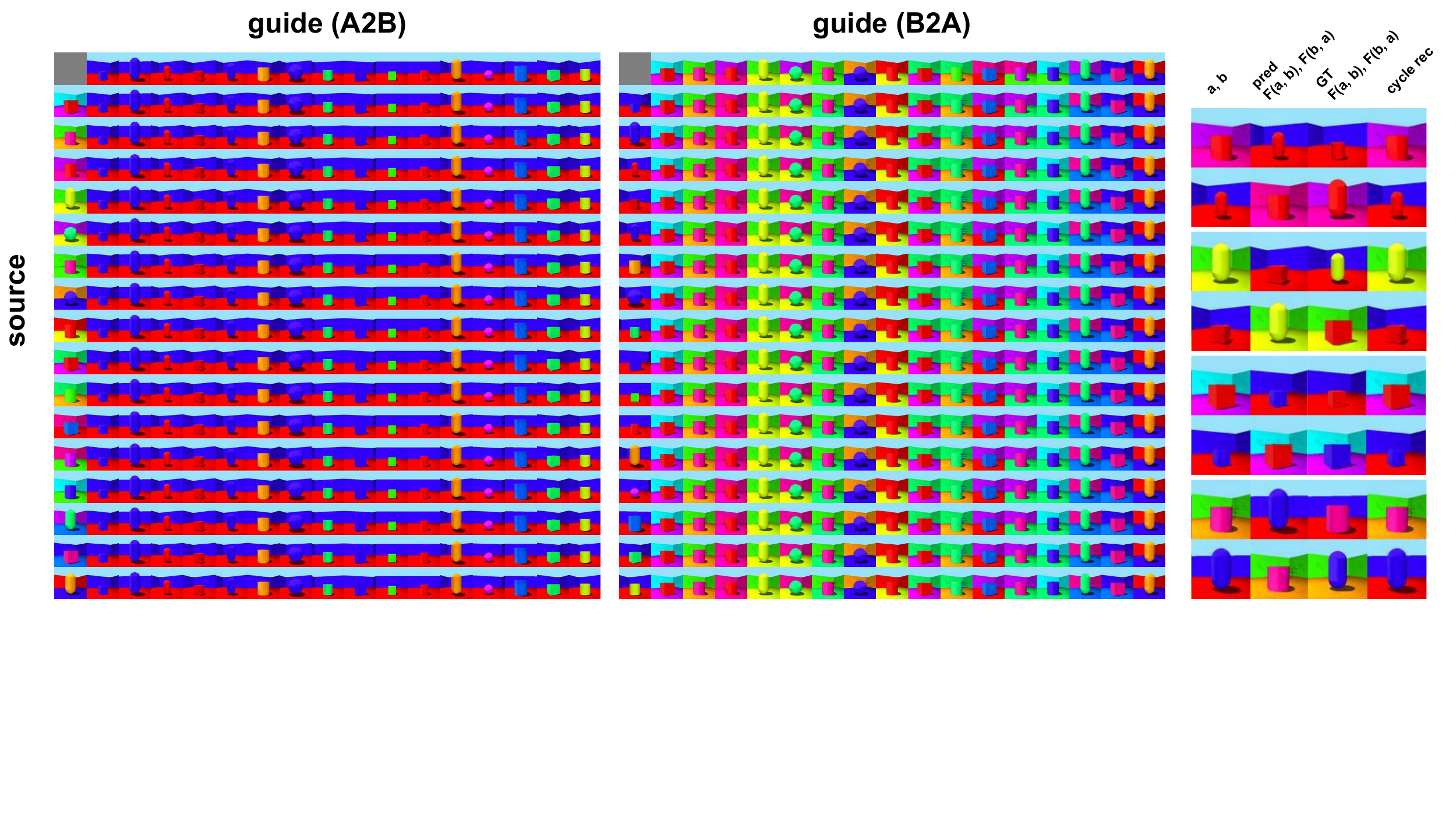}
\includegraphics[width=\linewidth,trim=0 2in 0 0,clip,page=2]{img/disent_ablation_large.pdf}
\end{center}
\caption{Additional \textbf{ablation visualizations} on Shapes-3D-A. Without capacity losses (top) the model always embeds the entire guidance image into the domain-specific embedding and ignores the source input. Without honesty losses (bottom) it ignores the guide input and embeds domain-specific information into the translated image to reconstruct it back to minimize the cycle reconstruction losses. }\label{fig:sup_ablations}
\end{figure*}
\begin{figure*}
\begin{center}
\includegraphics[width=\linewidth,trim=0 2.9in 5in 0,clip]{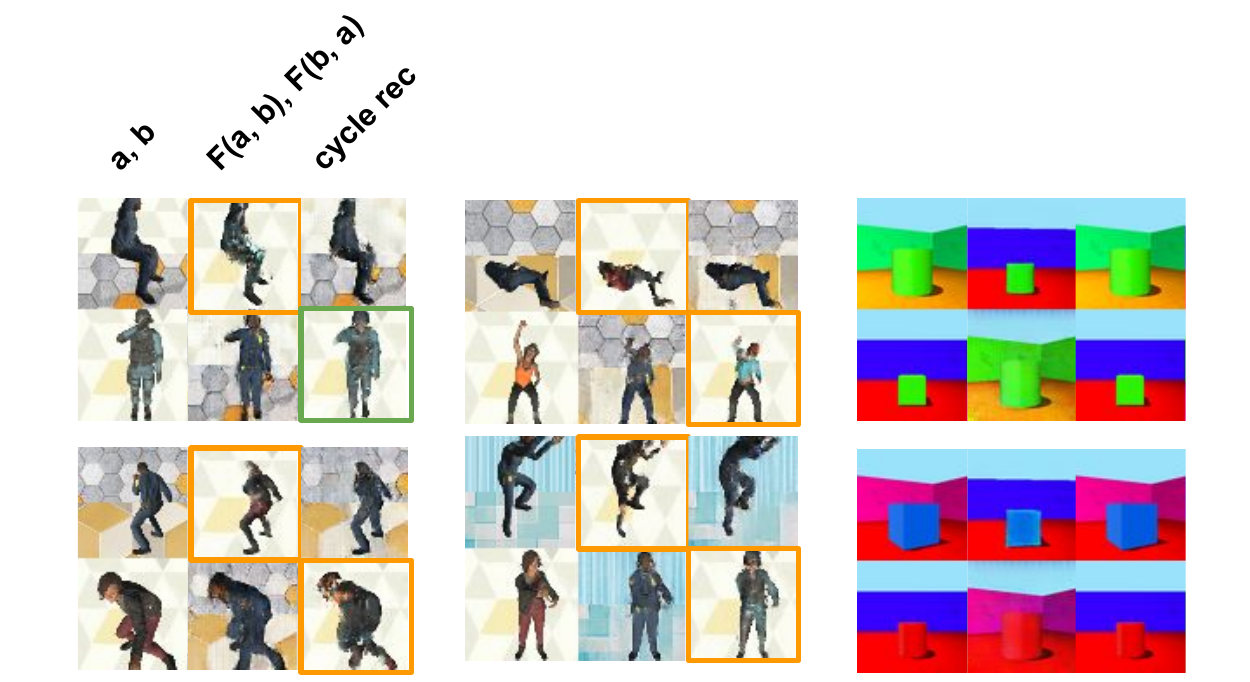}
\end{center}
\caption{Examples with \textbf{indistinguishable} attributes after translation cause instability in cycle reconstruction. For example, identities of actors with extreme poses are very dark and therefore hard to infer from translation results in SynAction; front-facing boxes and cylinders are almost indistinguishable too. }\label{fig:challenges_bad_pairs}
\end{figure*}
\begin{table*}[ht]
\inctabcolsep{-1pt} {
\begin{tabular}{lcccccccccccccccccc} \toprule
 & \multicolumn{6}{c}{Split A} & \multicolumn{6}{c}{Split B} & \multicolumn{6}{c}{Split C} \\
 \cmidrule(lr){2-7} \cmidrule(lr){8-13} \cmidrule(lr){14-19}
 & A & A & C & B & C & B & A & C & B & C & A & B & C & B & A & A & B & C \\
 & FC & WC & OC & SZ & SH & ORI & FC & WC & OC & SZ & SH & ORI & FC & WC & OC & SZ & SH & ORI \\ \midrule
MUNIT & 99 & 99 & 0 & 50 & 96 & 64 & 88 & 0 & 95 & 59 & 15 & 58 & 5 & 99 & 99 & 12 & 99 & 99 \\
MUNITX & 10 & 9 & 8 & 18 & 95 & 8 & 89 & 2 & 11 & 12 & 8 & 6 & 1 & 99 & 45 & 13 & 33 & 99 \\
DRIT & 13 & 16 & 10 & 14 & 7 & 95 & 10 & 9 & 7 & 27 & 8 & 6 & 7 & 21 & 12 & 13 & 22 & 42 \\
AugCycleGAN & 10 & 9 & 11 & 13 & 30 & 7 & 5 & 10 & 5 & 17 & 0 & 7 & 10 & 9 & 9 & 13 & 26 & 7 \\
StarGANv2 & 99 & 99 & 0 & 56 & 4 & 99 & 99 & 0 & 66 & 5 & 99 & 92 & 0 & 99 & 89 & 56 & 99 & 0 \\
DIDD & 99 & 27 & 72 & 12 & 87 & 8 & 62 & 29 & 10 & 41 & 59 & 59 & 38 & 17 & 25 & 28 & 27 & 48 \\ 
RIFT (ours) & 81 & 68 & 92 & 35 & 62 & 93 & 9 & 99 & 10 & 50 & 69 & 81 & 99 & 10 & 9 & 11 & 98 & 98 \\ \midrule
RAND & 10 & 9 & 10 & 12 & 24 & 6 & 10 & 10 & 9 & 12 & 25 & 6 & 10 & 10 & 10 & 25 & 12 & 6 \\ \bottomrule
\end{tabular}
}
\caption{Per-split attribute manipulation accuracy on 3D-Shapes-ABC with indicated attribute role (A, B - domain-specific; C - common/shared) across six attributes floor color (FC), wall color (WC), object color (OC), size (SZ), shape (SH), room orientation (ORI). \label{tab:non_agg_3dsh}}
\end{table*}

\clearpage

\newcommand*{\shapesacaption}{Qualitative comparison to existing methods on Shapes-3D-A across seven methods: Augmented CycleGAN \cite{almahairi2018augmented}, DRIT++ \cite{DRIT_plus}, DIDD \cite{benaim2019didd}, MUNIT \cite{huang2018multimodal}, MUNITX \cite{bashkirova2021evaluation}, StarGANv2 \cite{choi2020stargan} and the proposed RIFT. The rightmost column shows ground truth predictions for that domain pair.}

\pagestyle{empty}

\begin{figure*}
\begin{center}
\begin{picture}(385,230)
\put(0,0){\includegraphics[width=0.7777\linewidth]{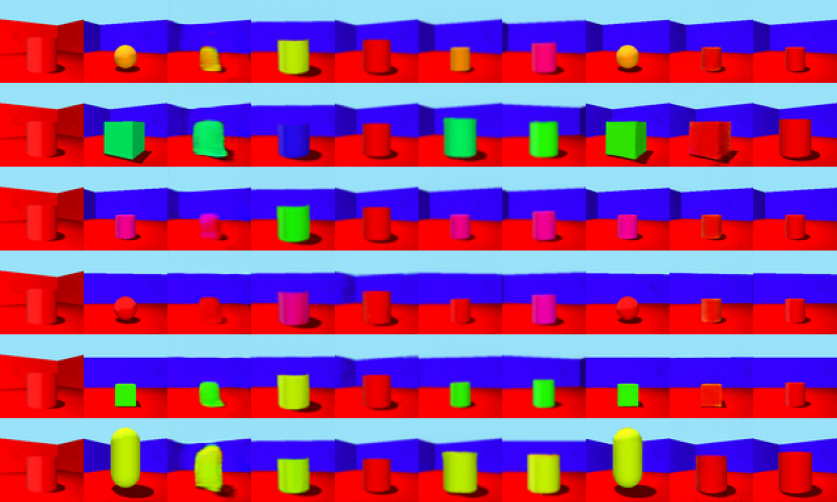}}
\put(10,235){input}
\put(39,235){guidance}
\put(82,235){AugCG}
\put(120,235){DRIT++}
\put(163,235){DIDD}
\put(196,235){MUNIT}
\put(233,235){\small MUNITX}
\put(272,235){\footnotesize StarGANv2}
\put(316,235){\bf RIFT}
\put(358,235){\bf GT}
\end{picture}

\vspace{50px}

\begin{picture}(385,230)
\put(0,0){\includegraphics[width=0.7777\linewidth]{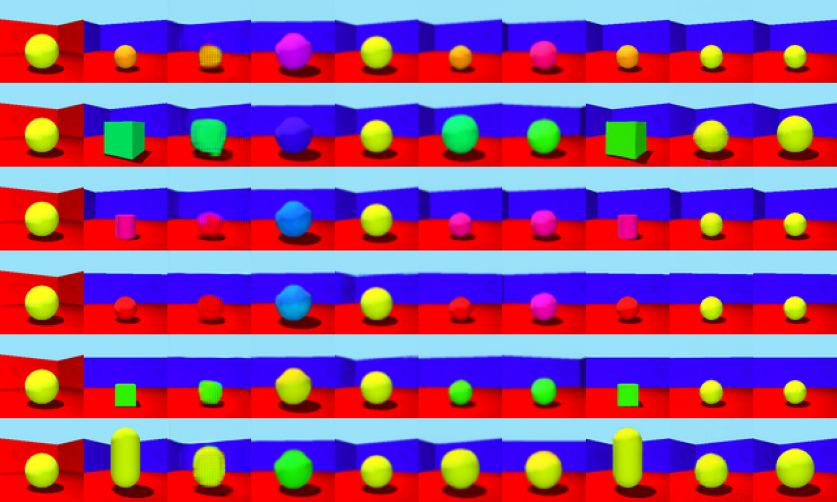}}
\put(10,235){input}
\put(39,235){guidance}
\put(82,235){AugCG}
\put(120,235){DRIT++}
\put(163,235){DIDD}
\put(196,235){MUNIT}
\put(233,235){\small MUNITX}
\put(272,235){\footnotesize StarGANv2}
\put(316,235){\bf RIFT}
\put(358,235){\bf GT}
\end{picture}
\end{center}

\caption{\shapesacaption}\label{fig:sup_example_18}
\end{figure*}

\begin{figure*}
\begin{center}
\begin{picture}(385,230)
\put(0,0){\includegraphics[width=0.7777\linewidth]{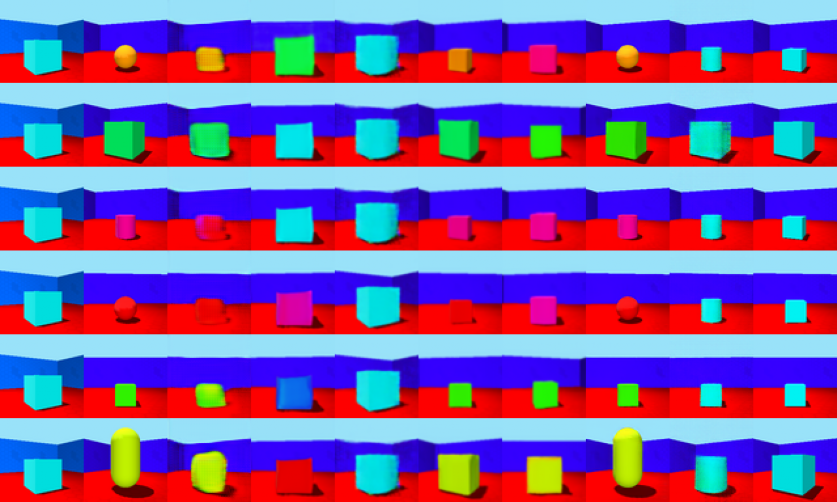}}
\put(10,235){input}
\put(39,235){guidance}
\put(82,235){AugCG}
\put(120,235){DRIT++}
\put(163,235){DIDD}
\put(196,235){MUNIT}
\put(233,235){\small MUNITX}
\put(272,235){\footnotesize StarGANv2}
\put(316,235){\bf RIFT}
\put(358,235){\bf GT}
\end{picture}

\vspace{50px}

\begin{picture}(385,230)
\put(0,0){\includegraphics[width=0.7777\linewidth]{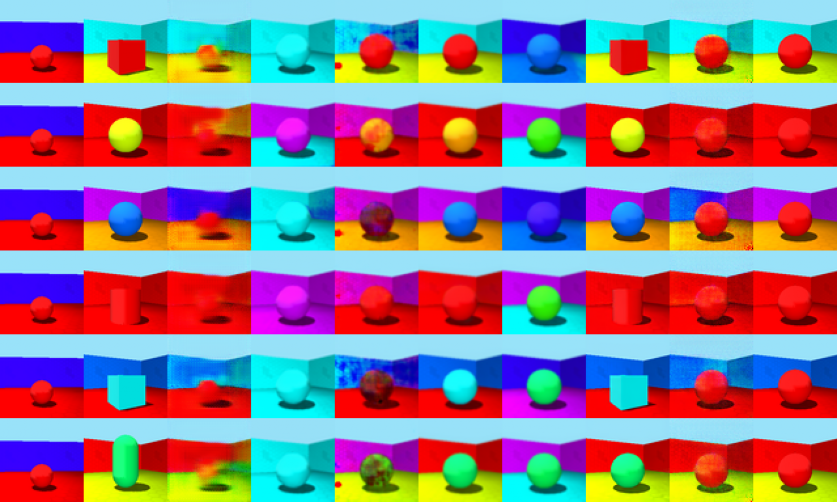}}
\put(10,235){input}
\put(39,235){guidance}
\put(82,235){AugCG}
\put(120,235){DRIT++}
\put(163,235){DIDD}
\put(196,235){MUNIT}
\put(233,235){\small MUNITX}
\put(272,235){\footnotesize StarGANv2}
\put(316,235){\bf RIFT}
\put(358,235){\bf GT}
\end{picture}
\end{center}

\caption{\shapesacaption}\label{fig:sup_example_19}
\end{figure*}

\begin{figure*}
\begin{center}
\begin{picture}(385,230)
\put(0,0){\includegraphics[width=0.7777\linewidth]{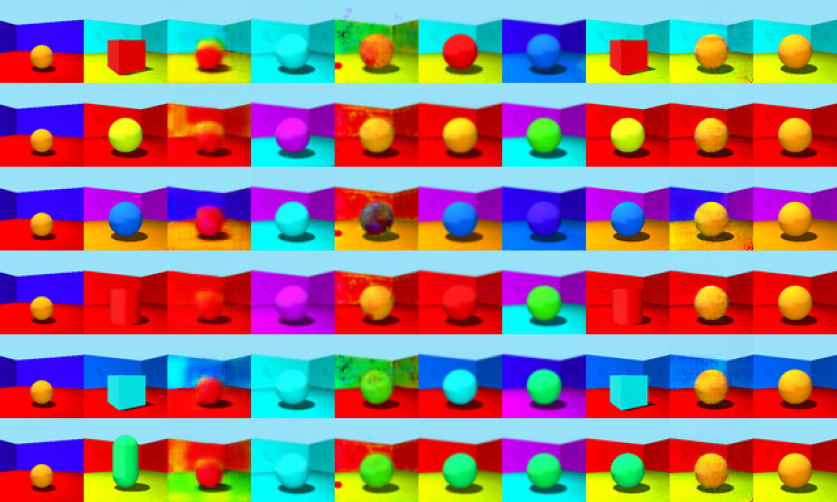}}
\put(10,235){input}
\put(39,235){guidance}
\put(82,235){AugCG}
\put(120,235){DRIT++}
\put(163,235){DIDD}
\put(196,235){MUNIT}
\put(233,235){\small MUNITX}
\put(272,235){\footnotesize StarGANv2}
\put(316,235){\bf RIFT}
\put(358,235){\bf GT}
\end{picture}

\vspace{50px}

\begin{picture}(385,230)
\put(0,0){\includegraphics[width=0.7777\linewidth]{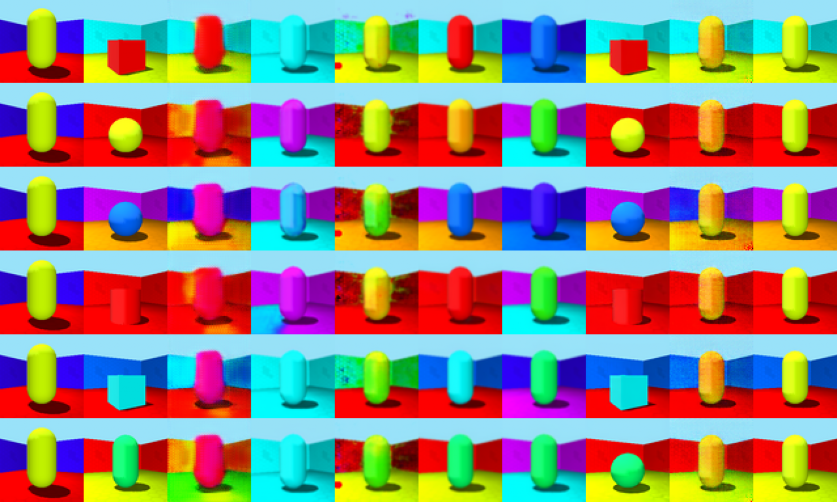}}
\put(10,235){input}
\put(39,235){guidance}
\put(82,235){AugCG}
\put(120,235){DRIT++}
\put(163,235){DIDD}
\put(196,235){MUNIT}
\put(233,235){\small MUNITX}
\put(272,235){\footnotesize StarGANv2}
\put(316,235){\bf RIFT}
\put(358,235){\bf GT}
\end{picture}
\end{center}

\caption{\shapesacaption}\label{fig:sup_example_20}
\end{figure*}

\begin{figure*}
\begin{center}
\begin{picture}(385,230)
\put(0,0){\includegraphics[width=0.7777\linewidth]{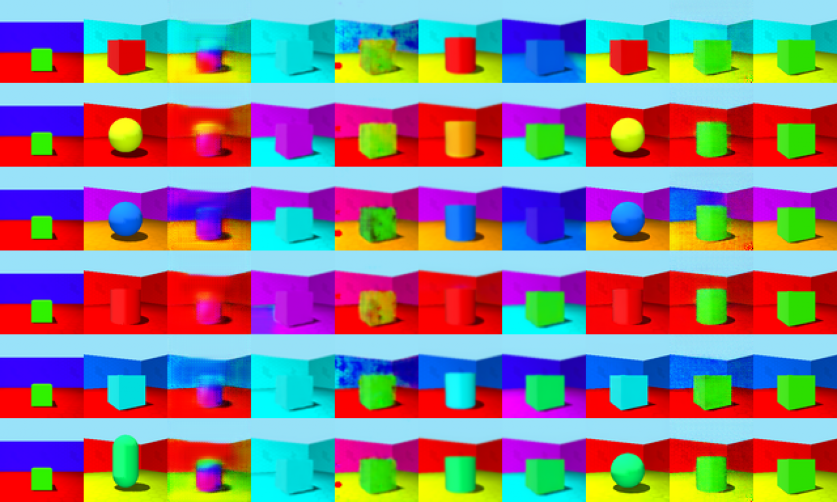}}
\put(10,235){input}
\put(39,235){guidance}
\put(82,235){AugCG}
\put(120,235){DRIT++}
\put(163,235){DIDD}
\put(196,235){MUNIT}
\put(233,235){\small MUNITX}
\put(272,235){\footnotesize StarGANv2}
\put(316,235){\bf RIFT}
\put(358,235){\bf GT}
\end{picture}

\vspace{50px}

\begin{picture}(385,230)
\put(0,0){\includegraphics[width=0.7777\linewidth]{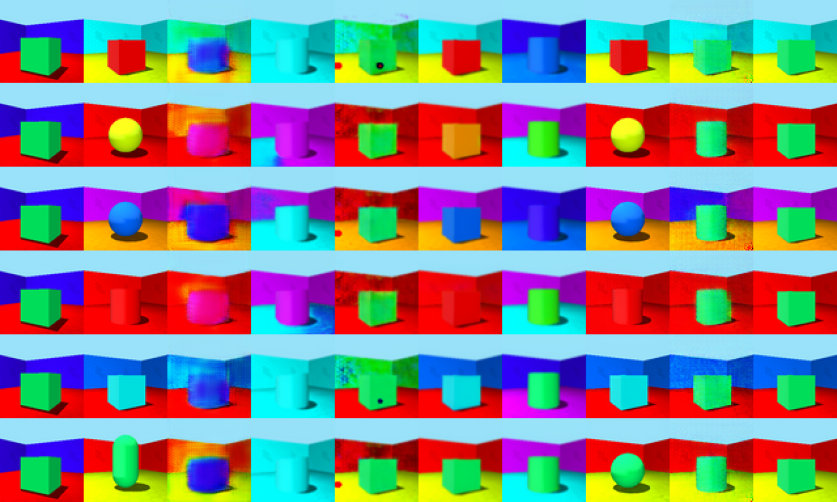}}
\put(10,235){input}
\put(39,235){guidance}
\put(82,235){AugCG}
\put(120,235){DRIT++}
\put(163,235){DIDD}
\put(196,235){MUNIT}
\put(233,235){\small MUNITX}
\put(272,235){\footnotesize StarGANv2}
\put(316,235){\bf RIFT}
\put(358,235){\bf GT}
\end{picture}
\end{center}

\caption{\shapesacaption}\label{fig:sup_example_21}
\end{figure*}

\begin{figure*}
\begin{center}
\begin{picture}(385,230)
\put(0,0){\includegraphics[width=0.7777\linewidth]{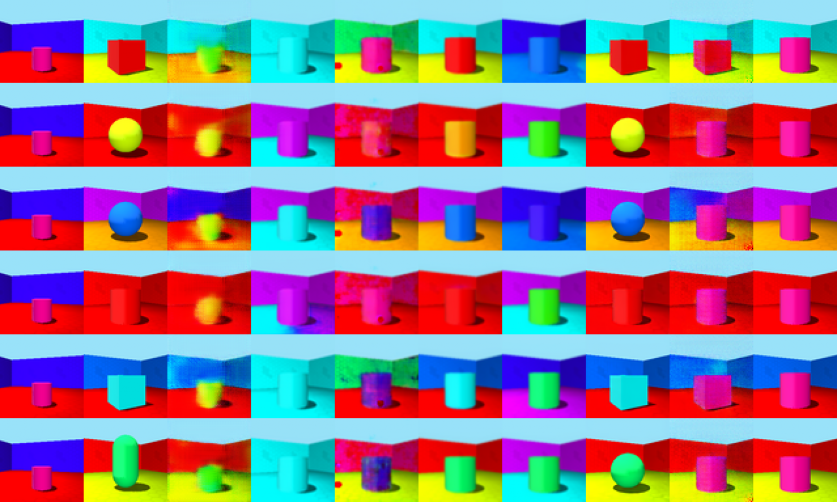}}
\put(10,235){input}
\put(39,235){guidance}
\put(82,235){AugCG}
\put(120,235){DRIT++}
\put(163,235){DIDD}
\put(196,235){MUNIT}
\put(233,235){\small MUNITX}
\put(272,235){\footnotesize StarGANv2}
\put(316,235){\bf RIFT}
\put(358,235){\bf GT}
\end{picture}

\vspace{50px}

\begin{picture}(385,230)
\put(0,0){\includegraphics[width=0.7777\linewidth]{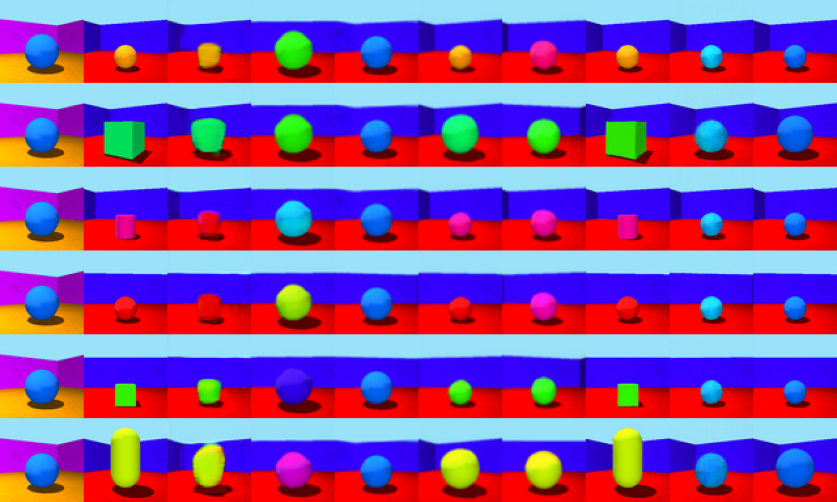}}
\put(10,235){input}
\put(39,235){guidance}
\put(82,235){AugCG}
\put(120,235){DRIT++}
\put(163,235){DIDD}
\put(196,235){MUNIT}
\put(233,235){\small MUNITX}
\put(272,235){\footnotesize StarGANv2}
\put(316,235){\bf RIFT}
\put(358,235){\bf GT}
\end{picture}
\end{center}

\caption{\shapesacaption}\label{fig:sup_example_22}
\end{figure*}

\begin{figure*}
\begin{center}
\begin{picture}(385,230)
\put(0,0){\includegraphics[width=0.7777\linewidth]{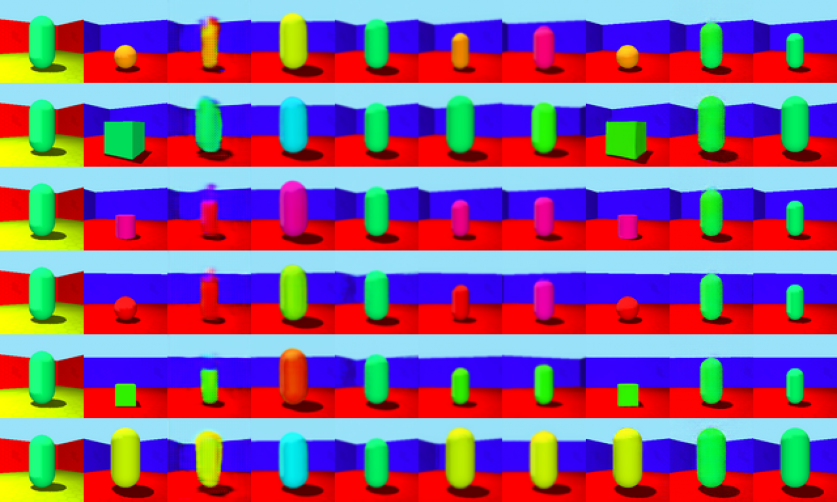}}
\put(10,235){input}
\put(39,235){guidance}
\put(82,235){AugCG}
\put(120,235){DRIT++}
\put(163,235){DIDD}
\put(196,235){MUNIT}
\put(233,235){\small MUNITX}
\put(272,235){\footnotesize StarGANv2}
\put(316,235){\bf RIFT}
\put(358,235){\bf GT}
\end{picture}

\vspace{50px}

\begin{picture}(385,230)
\put(0,0){\includegraphics[width=0.7777\linewidth]{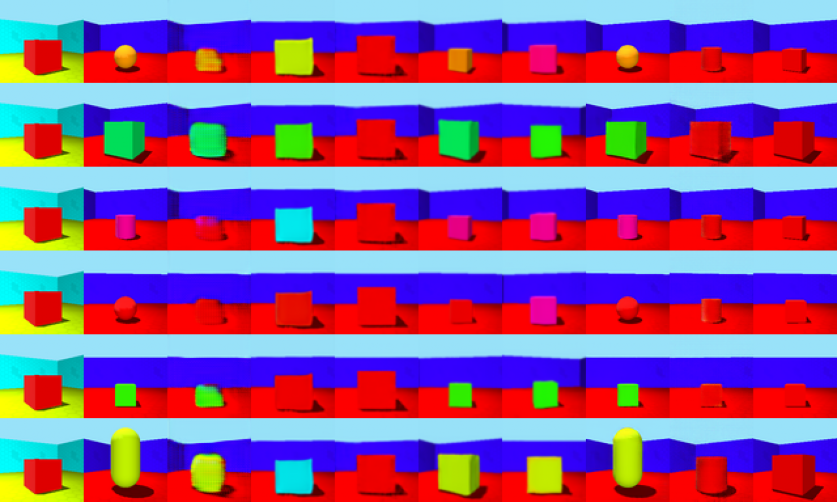}}
\put(10,235){input}
\put(39,235){guidance}
\put(82,235){AugCG}
\put(120,235){DRIT++}
\put(163,235){DIDD}
\put(196,235){MUNIT}
\put(233,235){\small MUNITX}
\put(272,235){\footnotesize StarGANv2}
\put(316,235){\bf RIFT}
\put(358,235){\bf GT}
\end{picture}
\end{center}

\caption{\shapesacaption}\label{fig:sup_example_23}
\end{figure*}

\newcommand*{\synactcaption}{Qualitative comparison to existing methods on SynAction across seven methods: Augmented CycleGAN \cite{almahairi2018augmented}, DRIT++ \cite{DRIT_plus}, DIDD \cite{benaim2019didd}, MUNIT \cite{huang2018multimodal}, MUNITX \cite{bashkirova2021evaluation}, StarGANv2 \cite{choi2020stargan} and the proposed RIFT. The rightmost column shows ground truth predictions for that domain pair.}

\begin{figure*}
\begin{center}
\begin{picture}(385,385)
\put(0,0){\includegraphics[width=0.7777\linewidth]{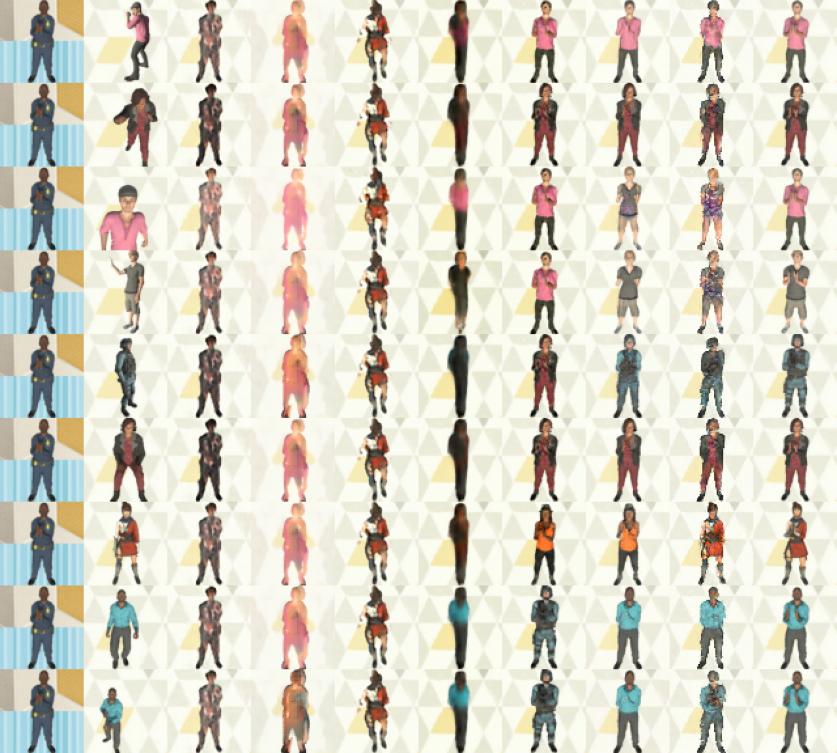}}
\put(10,355){input}
\put(39,355){guidance}
\put(82,355){AugCG}
\put(120,355){DRIT++}
\put(163,355){DIDD}
\put(196,355){MUNIT}
\put(233,355){\small MUNITX}
\put(272,355){\footnotesize StarGANv2}
\put(316,355){\bf RIFT}
\put(358,355){\bf GT}
\end{picture}

\end{center}

\caption{\synactcaption}\label{fig:sup_example_26}
\end{figure*}

\begin{figure*}
\begin{center}
\begin{picture}(385,385)
\put(0,0){\includegraphics[width=0.7777\linewidth]{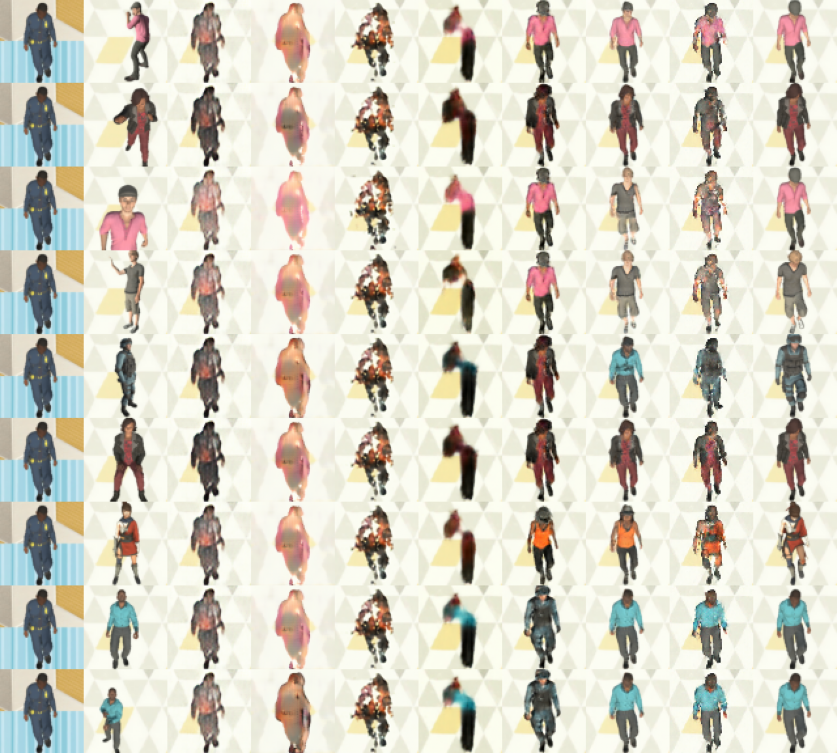}}
\put(10,355){input}
\put(39,355){guidance}
\put(82,355){AugCG}
\put(120,355){DRIT++}
\put(163,355){DIDD}
\put(196,355){MUNIT}
\put(233,355){\small MUNITX}
\put(272,355){\footnotesize StarGANv2}
\put(316,355){\bf RIFT}
\put(358,355){\bf GT}
\end{picture}

\end{center}

\caption{\synactcaption}\label{fig:sup_example_27}
\end{figure*}

\begin{figure*}
\begin{center}
\begin{picture}(385,385)
\put(0,0){\includegraphics[width=0.7777\linewidth]{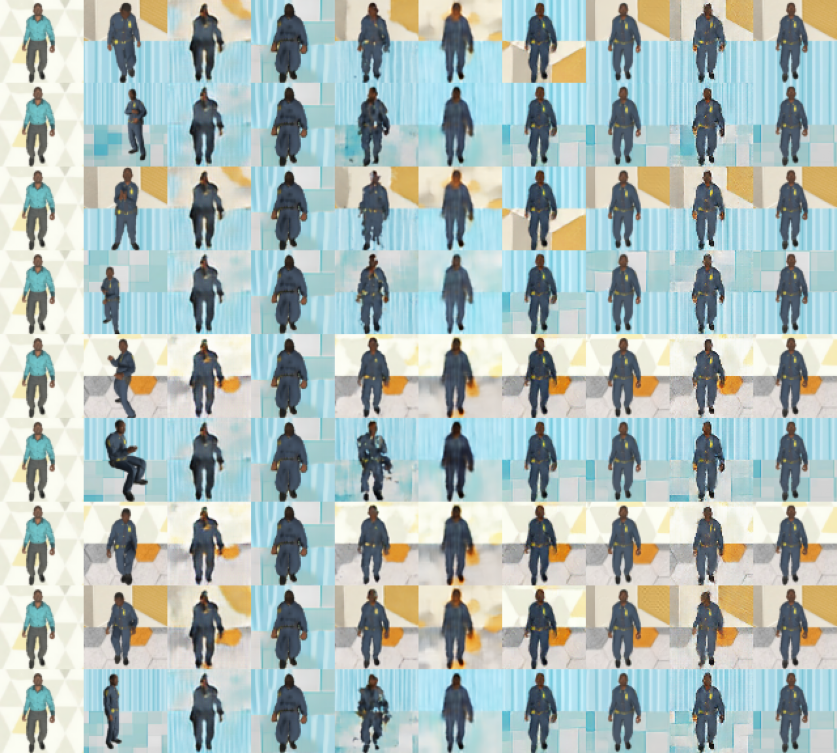}}
\put(10,355){input}
\put(39,355){guidance}
\put(82,355){AugCG}
\put(120,355){DRIT++}
\put(163,355){DIDD}
\put(196,355){MUNIT}
\put(233,355){\small MUNITX}
\put(272,355){\footnotesize StarGANv2}
\put(316,355){\bf RIFT}
\put(358,355){\bf GT}
\end{picture}

\end{center}

\caption{\synactcaption}\label{fig:sup_example_28}
\end{figure*}

\begin{figure*}
\begin{center}
\begin{picture}(385,385)
\put(0,0){\includegraphics[width=0.7777\linewidth]{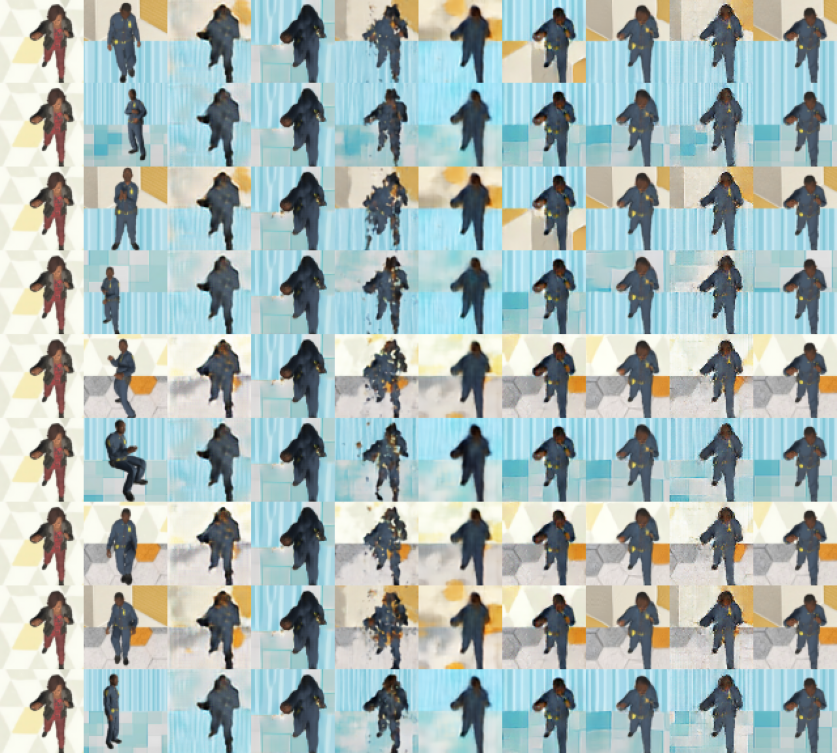}}
\put(10,355){input}
\put(39,355){guidance}
\put(82,355){AugCG}
\put(120,355){DRIT++}
\put(163,355){DIDD}
\put(196,355){MUNIT}
\put(233,355){\small MUNITX}
\put(272,355){\footnotesize StarGANv2}
\put(316,355){\bf RIFT}
\put(358,355){\bf GT}
\end{picture}

\end{center}

\caption{\synactcaption}\label{fig:sup_example_29}
\end{figure*}

\begin{figure*}
\begin{center}
\begin{picture}(385,385)
\put(0,0){\includegraphics[width=0.7777\linewidth]{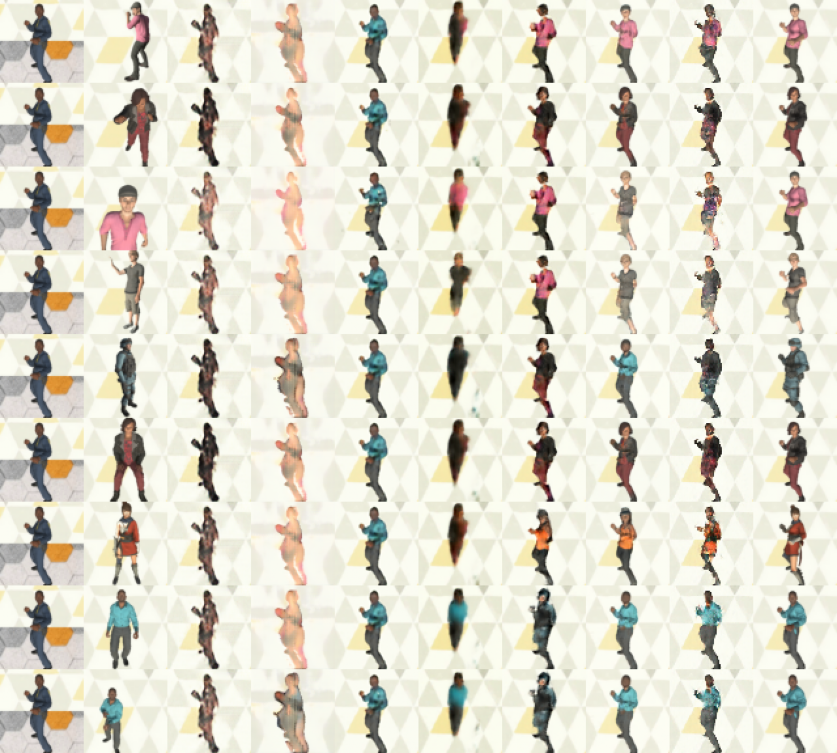}}
\put(10,355){input}
\put(39,355){guidance}
\put(82,355){AugCG}
\put(120,355){DRIT++}
\put(163,355){DIDD}
\put(196,355){MUNIT}
\put(233,355){\small MUNITX}
\put(272,355){\footnotesize StarGANv2}
\put(316,355){\bf RIFT}
\put(358,355){\bf GT}
\end{picture}

\end{center}

\caption{\synactcaption}\label{fig:sup_example_30}
\end{figure*}

\begin{figure*}
\begin{center}
\begin{picture}(385,385)
\put(0,0){\includegraphics[width=0.7777\linewidth]{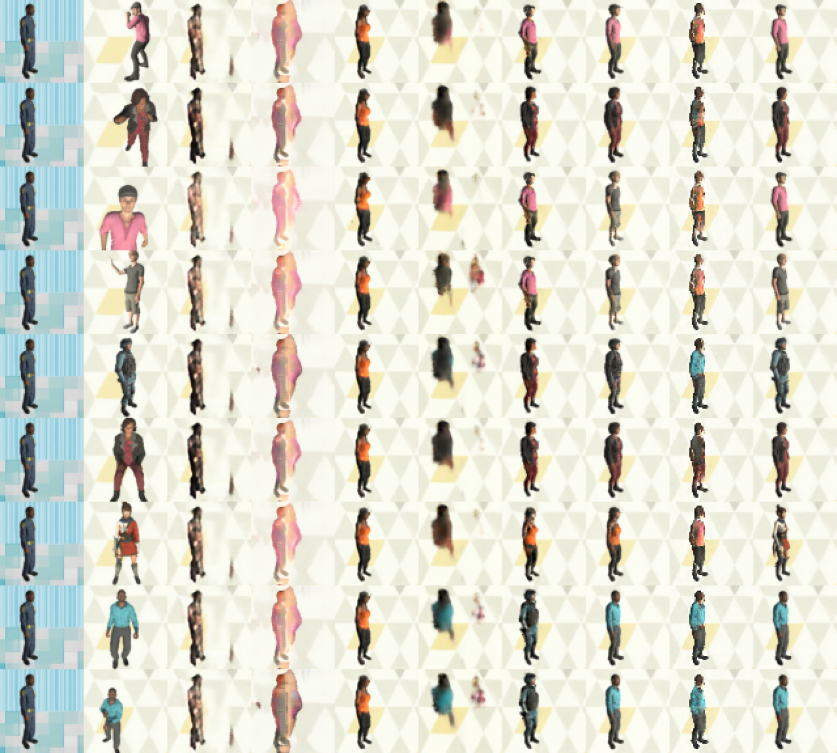}}
\put(10,355){input}
\put(39,355){guidance}
\put(82,355){AugCG}
\put(120,355){DRIT++}
\put(163,355){DIDD}
\put(196,355){MUNIT}
\put(233,355){\small MUNITX}
\put(272,355){\footnotesize StarGANv2}
\put(316,355){\bf RIFT}
\put(358,355){\bf GT}
\end{picture}

\end{center}

\caption{\synactcaption}\label{fig:sup_example_31}
\end{figure*}

\begin{figure*}
\begin{center}
\begin{picture}(385,385)
\put(0,0){\includegraphics[width=0.7777\linewidth]{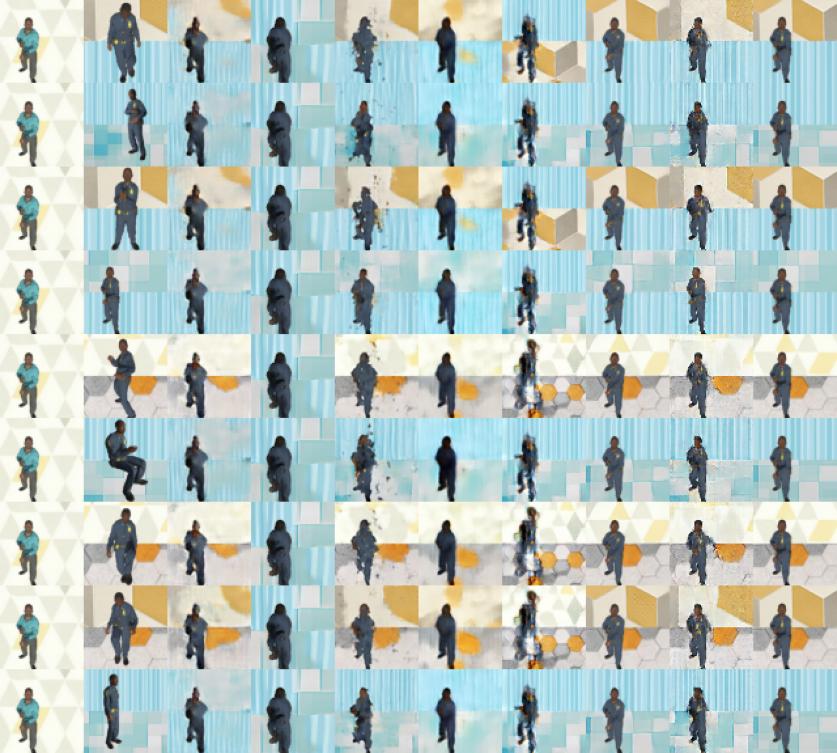}}
\put(10,355){input}
\put(39,355){guidance}
\put(82,355){AugCG}
\put(120,355){DRIT++}
\put(163,355){DIDD}
\put(196,355){MUNIT}
\put(233,355){\small MUNITX}
\put(272,355){\footnotesize StarGANv2}
\put(316,355){\bf RIFT}
\put(358,355){\bf GT}
\end{picture}

\end{center}

\caption{\synactcaption}\label{fig:sup_example_32}
\end{figure*}

\newcommand*{\celebacaption}{Qualitative comparison to existing methods on Celeb-A across seven methods: Augmented CycleGAN \cite{almahairi2018augmented}, DRIT++ \cite{DRIT_plus}, DIDD \cite{benaim2019didd}, MUNIT \cite{huang2018multimodal}, MUNITX \cite{bashkirova2021evaluation}, StarGANv2 \cite{choi2020stargan} and the proposed RIFT.}

\pagestyle{empty}

\begin{figure*}
\begin{center}
\begin{picture}(345,630)
\put(0,0){\includegraphics[width=0.7\linewidth]{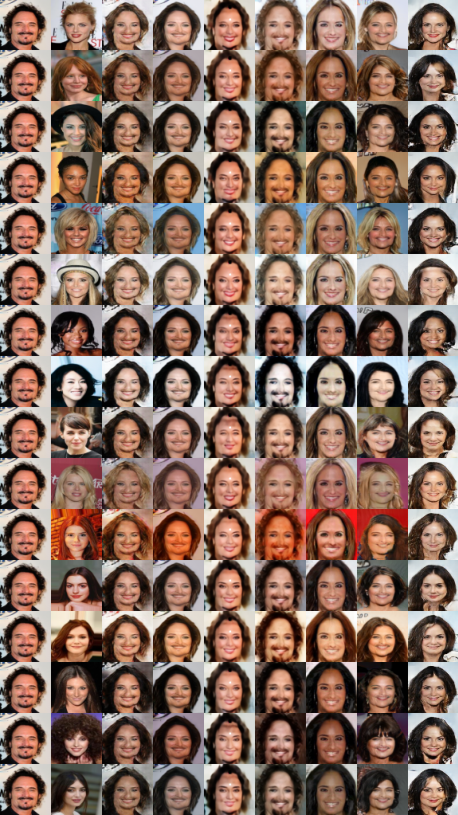}}
\put(10,623){input}
\put(39,623){guidance}
\put(82,623){AugCG}
\put(120,623){DRIT++}
\put(163,623){DIDD}
\put(196,623){MUNIT}
\put(233,623){\small MUNITX}
\put(272,623){\footnotesize StarGANv2}
\put(316,623){\bf RIFT}
\end{picture}
\end{center}
\caption{\celebacaption}\label{fig:sup_example_1}
\end{figure*}

\begin{figure*}
\begin{center}
\begin{picture}(345,630)
\put(0,0){\includegraphics[width=0.7\linewidth]{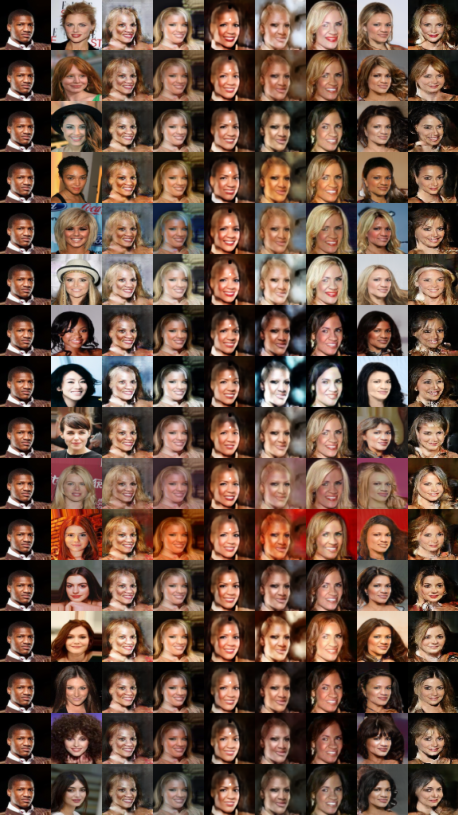}}
\put(10,623){input}
\put(39,623){guidance}
\put(82,623){AugCG}
\put(120,623){DRIT++}
\put(163,623){DIDD}
\put(196,623){MUNIT}
\put(233,623){\small MUNITX}
\put(272,623){\footnotesize StarGANv2}
\put(316,623){\bf RIFT}
\end{picture}
\end{center}
\caption{\celebacaption}\label{fig:sup_example_3}
\end{figure*}

\begin{figure*}
\begin{center}
\begin{picture}(345,630)
\put(0,0){\includegraphics[width=0.7\linewidth]{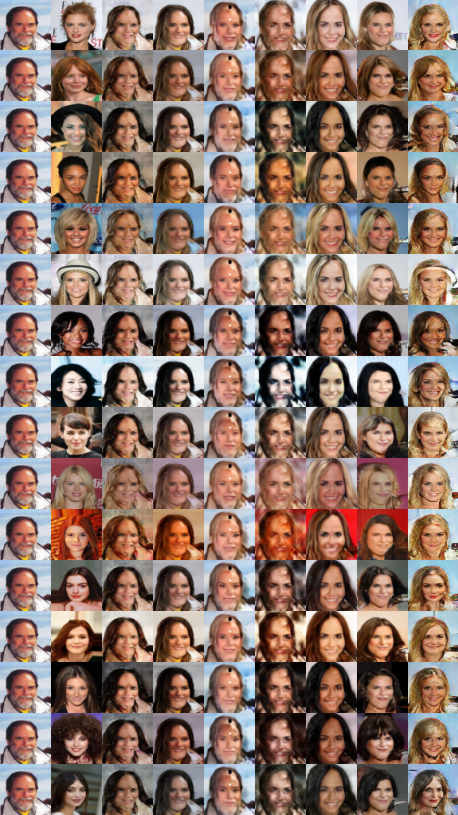}}
\put(10,623){input}
\put(39,623){guidance}
\put(82,623){AugCG}
\put(120,623){DRIT++}
\put(163,623){DIDD}
\put(196,623){MUNIT}
\put(233,623){\small MUNITX}
\put(272,623){\footnotesize StarGANv2}
\put(316,623){\bf RIFT}
\end{picture}
\end{center}
\caption{\celebacaption}\label{fig:sup_example_4}
\end{figure*}

\begin{figure*}
\begin{center}
\begin{picture}(345,630)
\put(0,0){\includegraphics[width=0.7\linewidth]{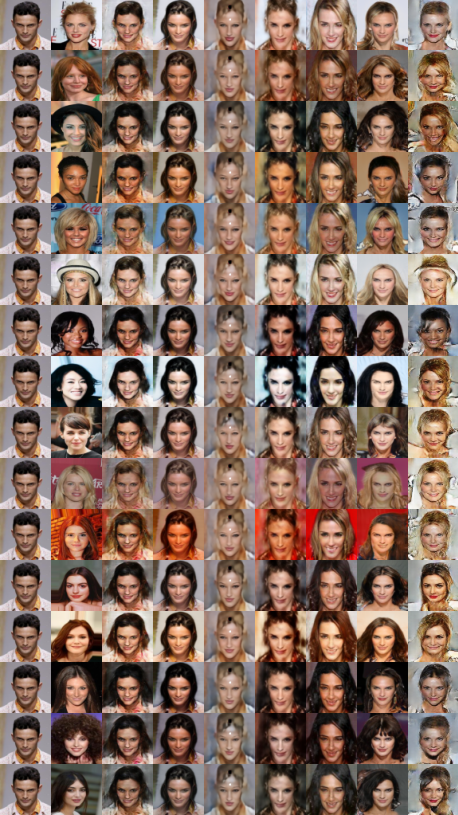}}
\put(10,623){input}
\put(39,623){guidance}
\put(82,623){AugCG}
\put(120,623){DRIT++}
\put(163,623){DIDD}
\put(196,623){MUNIT}
\put(233,623){\small MUNITX}
\put(272,623){\footnotesize StarGANv2}
\put(316,623){\bf RIFT}
\end{picture}
\end{center}
\caption{\celebacaption}\label{fig:sup_example_5}
\end{figure*}

\begin{figure*}
\begin{center}
\begin{picture}(345,630)
\put(0,0){\includegraphics[width=0.7\linewidth]{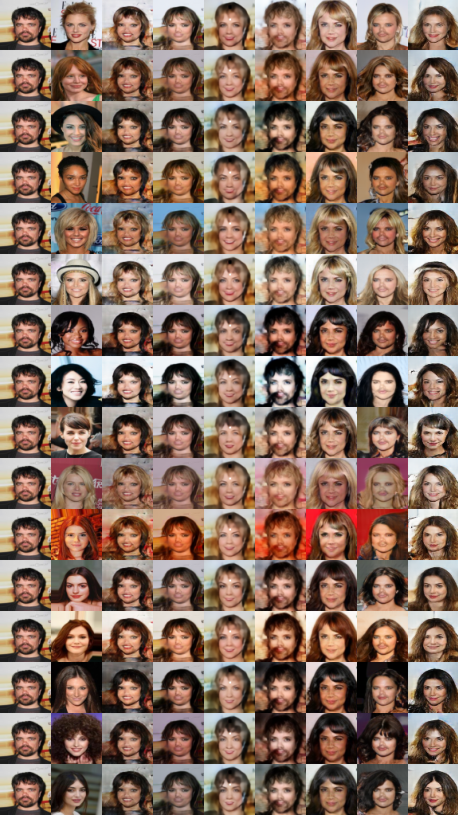}}
\put(10,623){input}
\put(39,623){guidance}
\put(82,623){AugCG}
\put(120,623){DRIT++}
\put(163,623){DIDD}
\put(196,623){MUNIT}
\put(233,623){\small MUNITX}
\put(272,623){\footnotesize StarGANv2}
\put(316,623){\bf RIFT}
\end{picture}
\end{center}
\caption{\celebacaption}\label{fig:sup_example_6}
\end{figure*}

\begin{figure*}
\begin{center}
\begin{picture}(345,630)
\put(0,0){\includegraphics[width=0.7\linewidth]{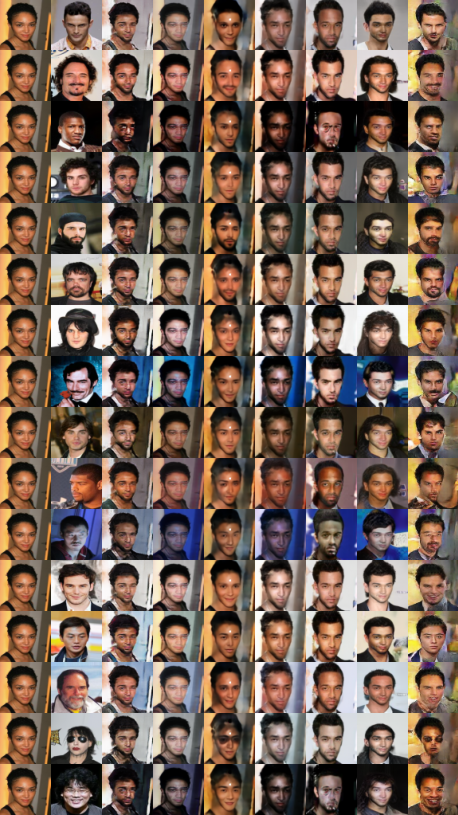}}
\put(10,623){input}
\put(39,623){guidance}
\put(82,623){AugCG}
\put(120,623){DRIT++}
\put(163,623){DIDD}
\put(196,623){MUNIT}
\put(233,623){\small MUNITX}
\put(272,623){\footnotesize StarGANv2}
\put(316,623){\bf RIFT}
\end{picture}
\end{center}
\caption{\celebacaption}\label{fig:sup_example_7}
\end{figure*}

\begin{figure*}
\begin{center}
\begin{picture}(345,630)
\put(0,0){\includegraphics[width=0.7\linewidth]{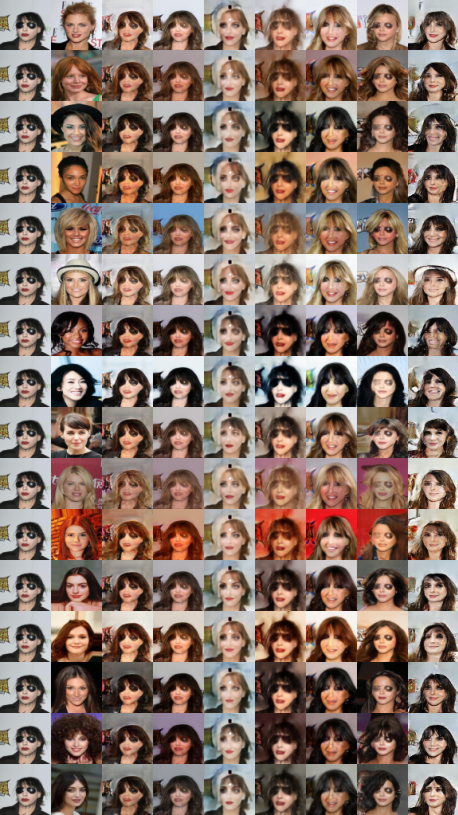}}
\put(10,623){input}
\put(39,623){guidance}
\put(82,623){AugCG}
\put(120,623){DRIT++}
\put(163,623){DIDD}
\put(196,623){MUNIT}
\put(233,623){\small MUNITX}
\put(272,623){\footnotesize StarGANv2}
\put(316,623){\bf RIFT}
\end{picture}
\end{center}
\caption{\celebacaption}\label{fig:sup_example_8}
\end{figure*}

\begin{figure*}
\begin{center}
\begin{picture}(345,630)
\put(0,0){\includegraphics[width=0.7\linewidth]{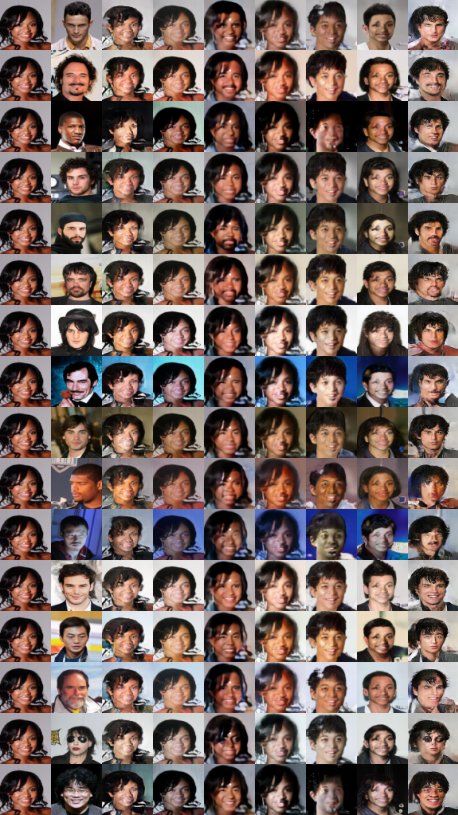}}
\put(10,623){input}
\put(39,623){guidance}
\put(82,623){AugCG}
\put(120,623){DRIT++}
\put(163,623){DIDD}
\put(196,623){MUNIT}
\put(233,623){\small MUNITX}
\put(272,623){\footnotesize StarGANv2}
\put(316,623){\bf RIFT}
\end{picture}
\end{center}
\caption{\celebacaption}\label{fig:sup_example_9}
\end{figure*}

\begin{figure*}
\begin{center}
\begin{picture}(345,630)
\put(0,0){\includegraphics[width=0.7\linewidth]{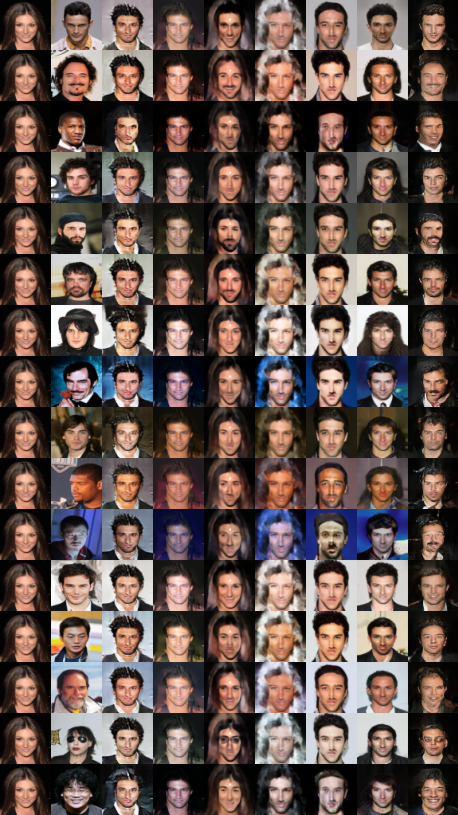}}
\put(10,623){input}
\put(39,623){guidance}
\put(82,623){AugCG}
\put(120,623){DRIT++}
\put(163,623){DIDD}
\put(196,623){MUNIT}
\put(233,623){\small MUNITX}
\put(272,623){\footnotesize StarGANv2}
\put(316,623){\bf RIFT}
\end{picture}
\end{center}
\caption{\celebacaption}\label{fig:sup_example_12}
\end{figure*}

\begin{figure*}
\begin{center}
\begin{picture}(345,630)
\put(0,0){\includegraphics[width=0.7\linewidth]{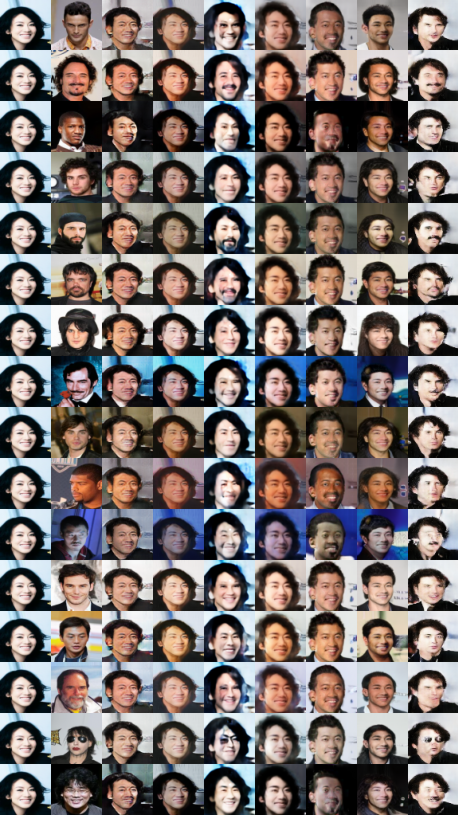}}
\put(10,623){input}
\put(39,623){guidance}
\put(82,623){AugCG}
\put(120,623){DRIT++}
\put(163,623){DIDD}
\put(196,623){MUNIT}
\put(233,623){\small MUNITX}
\put(272,623){\footnotesize StarGANv2}
\put(316,623){\bf RIFT}
\end{picture}
\end{center}
\caption{\celebacaption}\label{fig:sup_example_13}
\end{figure*}

\begin{figure*}
\begin{center}
\begin{picture}(345,630)
\put(0,0){\includegraphics[width=0.7\linewidth]{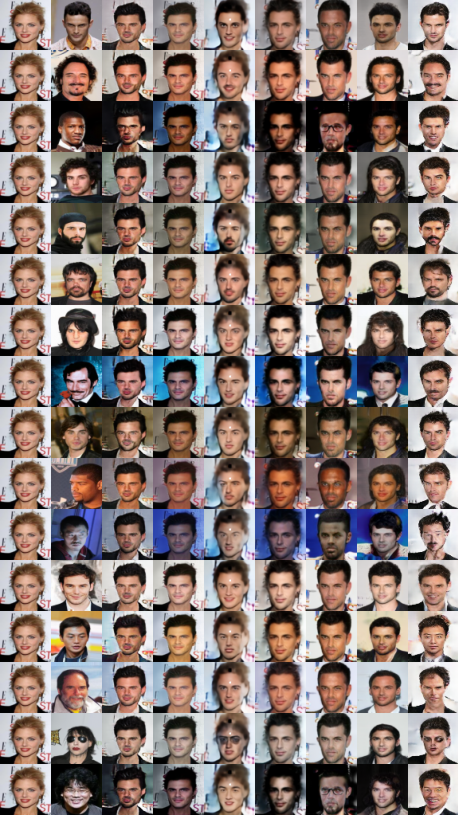}}
\put(10,623){input}
\put(39,623){guidance}
\put(82,623){AugCG}
\put(120,623){DRIT++}
\put(163,623){DIDD}
\put(196,623){MUNIT}
\put(233,623){\small MUNITX}
\put(272,623){\footnotesize StarGANv2}
\put(316,623){\bf RIFT}
\end{picture}
\end{center}
\caption{\celebacaption}\label{fig:sup_example_15}
\end{figure*}

\end{document}